\documentclass[journal]{IEEEtran}
\usepackage{amsmath,amsfonts}
\usepackage{algorithm}
\usepackage[noend]{algpseudocode}
\usepackage{hyperref}
\usepackage[capitalize]{cleveref}
\usepackage[numbers]{natbib}
\usepackage{booktabs}
\usepackage{graphicx}
\usepackage{epstopdf}
\usepackage{amsthm}
\usepackage{makecell}
\usepackage[font=footnotesize]{subcaption}
\usepackage{xstring}
\usepackage{stfloats}
\newcommand{\subfigref}[1]{
    \StrBefore{#1}{-}[\mainfig]
    Fig.~\ref{\mainfig}(\subref{#1})
}
\makeatletter
\let\MYcaption\@makecaption
\makeatother
\makeatletter
\let\@makecaption\MYcaption
\makeatother
\allowdisplaybreaks
\usepackage{breqn}
\usepackage{longtable}

\newtheorem{theorem}{Theorem}

\ifCLASSINFOpdf
\else
\fi
\begin{document}
    \title{Enhancing Generalization in Evolutionary Feature Construction for Symbolic Regression through Vicinal Jensen Gap Minimization}
    \author{
        Hengzhe Zhang,
        Qi Chen,~\IEEEmembership{Member,~IEEE,}
        Bing Xue,~\IEEEmembership{Fellow,~IEEE,}
        Wolfgang Banzhaf,~\IEEEmembership{Member,~IEEE,}
        Mengjie Zhang,~\IEEEmembership{Fellow,~IEEE}
        \thanks{
            H. Zhang, Q. Chen, B. Xue and M. Zhang are with the Centre for Data Science and Artificial Intelligence \& School of Engineering and Computer Science, Victoria University of Wellington, PO Box 600, Wellington 6140, New Zealand (e-mails: hengzhe.zhang@ecs.vuw.ac.nz; qi.chen@ecs.vuw.ac.nz; bing.xue@ecs.vuw.ac.nz; and mengjie.zhang@ecs.vuw.ac.nz).\\
            \indent W. Banzhaf is with the Department of Computer Science and Engineering, Michigan State University, East Lansing, MI 48824, USA (e-mails: banzhafw@msu.edu).
        }
        \thanks{This work was supported in part by the Marsden Fund of New Zealand Government under Contract VUW1913, and Contract VUW2016; in part by the Science for Technological Innovation Challenge (SfTI) Fund under Grant E3603/2903; in part by the MBIE Data Science SSIF Fund under Contract RTVU1914; and in part by the MBIE Endeavor Research Programme under Contract C11X2001 and Contract UOCX2104.}
    }
    \markboth{Journal of \LaTeX\ Class Files,~Vol.~14, No.~8, August~2015}
    {Shell \MakeLowercase{\textit{et al.}}: Bare Demo of IEEEtran.cls for IEEE Journals}
    \maketitle
        \begin{abstract}
            Genetic programming-based feature construction has achieved significant success in recent years as an automated machine learning technique to enhance learning performance. However, overfitting remains a challenge that limits its broader applicability. To improve generalization, we prove that vicinal risk, estimated through noise perturbation or mixup-based data augmentation, is bounded by the sum of empirical risk and a regularization term—either finite difference or the vicinal Jensen gap. Leveraging this decomposition, we propose an evolutionary feature construction framework that jointly optimizes empirical risk and the vicinal Jensen gap to control overfitting. Since datasets may vary in noise levels, we develop a noise estimation strategy to dynamically adjust regularization strength. Furthermore, to mitigate manifold intrusion—where data augmentation may generate unrealistic samples that fall outside the data manifold—we propose a manifold intrusion detection mechanism. Experimental results on 58 datasets demonstrate the effectiveness of Jensen gap minimization compared to other complexity measures. Comparisons with 15 machine learning algorithms further indicate that genetic programming with the proposed overfitting control strategy achieves superior performance.
        \end{abstract}
        \IEEEpeerreviewmaketitle
        \section{Introduction}
        \label{sec: Introduction}
        Automated feature construction is a popular technique in the domain of automated machine learning (AutoML)~\cite{EF-TEVC}. For a given dataset $(X,Y)$, automated feature construction aims to enhance the learning performance of a machine learning algorithm $\mathcal{A}$ by constructing a set of high-level features $\phi_1, \dots, \phi_m$  from the original (low-level) feature space $X$. Unlike deep learning, which relies on numerous parameters to construct features, GP-based feature construction focuses on creating symbolic models to construct features in function space to improve the performance of learning algorithms.
        Genetic Programming (GP), an evolutionary computation technique, provides key advantages such as interpretable representations, gradient-free search, and global optimization capabilities. These strengths make GP particularly effective for feature construction in non-differentiable symbolic spaces and for optimizing non-differentiable loss functions. However, a significant challenge with modern evolutionary feature construction methods, including GP, is overfitting. While these methods can achieve impressive performance on training sets, the features they construct might not generalize well to unseen data, mirroring the challenges faced by many automatic feature construction methods, such as deep neural networks and kernel methods.
        The issue of overfitting in GP has been widely studied. A well-recognized approach to controlling model complexity is by regulating model size~\cite{la2020learning}.
        However, growing evidence suggests that model size alone is insufficient to fully characterize model complexity~\cite{vanneschi2010measuring}. For example, using model size as a complexity measure cannot differentiate the complexity between $x_1 \times x_2$ and $\sin(\sin(x))$. Thus, incorporating semantic considerations is crucial.
        Building on this insight, several studies in GP address overfitting from the perspective of statistical machine learning theory~\cite{chen2018structural,chen2020rademacher}. Nonetheless, traditional metrics such as VC-dimension~\cite{chen2018structural} or Rademacher complexity~\cite{chen2020rademacher} often fall short in accurately estimating generalization performance when applied to deep-learning-based feature construction techniques. With GP being recognized as a data-efficient deep learning technique~\cite{bi2022genetic, wu2023neural}, it is essential to explore more effective methods for controlling overfitting.
        In the realm of deep learning, vicinal risk minimization (VRM)~\cite{chapelle2000vicinal} has recently garnered significant success. The key idea of VRM is to not only minimize the training loss on the original training samples but also on the vicinal samples surrounding each training instance, which are synthesized through data augmentation methods in practice~\cite{zhang2018mixup}. The intuition behind VRM is that by minimizing the empirical loss around the neighborhood $\mathcal{N}(x_i)$ of each training instance $x_i$, the model can maintain its behavior when test data slightly deviates from the training data. Formally, for a loss function $\mathcal{L}$, VRM optimizes the following objective function:
        \begin{equation}
            \mathcal{V}(f) = \frac{1}{n} \sum_{i=1}^n \int \mathcal{L}(f(x), y) \, dP_{x_i}(x),
            \label{eq: General VRM}
        \end{equation}
        where $x$ denotes a vicinal example that is a nearby datapoint of a training example $x_i$ on the data manifold, and $P_{x_i}(x)$ represents the probability distribution of vicinal examples $x$ in the neighborhood of $x_i$. Intuitively, vicinal examples $x$ that are further from the training data $x_i$ have lower probabilities $P_{x_i}(x)$.
        In deep learning, vicinal risk has been widely adopted as an optimization objective to mitigate overfitting. However, traditional VRM does not explicitly disentangle the empirical loss and the regularization term from vicinal risk, instead optimizing both jointly by minimizing the empirical loss on synthesized samples from the neighborhoods of the original training data~\cite{zhang2018mixup}. To adapt VRM for evolutionary feature construction, this paper proposes decomposing the vicinal risk $\mathcal{V}(f)$ into an empirical loss component and a regularization term to more effectively control overfitting, motivated by several key considerations:
        \begin{itemize}
            \item \textbf{Flexible Trade-off between Accuracy and Smoothness:} Decomposing vicinal risk into an empirical loss component, $\frac{1}{n} \sum_{i=1}^n \int \mathcal{L} \left(f(x_i), y_i\right) \, dP(x)$, and a regularization term allows for fine-grained control over the balance between training loss and model smoothness. This flexible balance enables a more nuanced trade-off between empirical loss and regularization.
            \item \textbf{Accurate Training Error:} In \cref{eq: General VRM}, the integral in $\mathcal{V}(f)$ is typically estimated using samples from the neighborhood $\mathcal{N}(x_i)$, which can lead to inaccuracies due to limited sampling. However, empirical loss can be computed accurately on the original training data. Decomposing vicinal risk helps mitigate these inaccuracies by ensuring that the empirical loss, i.e., $\frac{1}{n} \sum_{i=1}^n \mathcal{L} \left(f(x_i), y_i\right)$, is accurately calculated.
            \item \textbf{Optimizing Validation Loss:} Traditional vicinal loss entails empirical loss~\cite{zhang2018mixup}. However, if we intend to replace empirical loss with alternative loss functions, such as cross-validation loss—which is widely used in feature construction and AutoML—the decomposition of the vicinal loss becomes necessary.
        \end{itemize}
        To decompose vicinal risk, this paper proves that vicinal risk is bounded by the sum of two terms: empirical loss and either a finite difference or a vicinal Jensen gap. Specifically, if vicinal examples are synthesized via noise perturbation then a finite difference is the regularization term. If synthesized using the mixup method~\cite{zhang2018mixup}, the vicinal Jensen gap becomes the regularization term. Mixup is a data augmentation technique that linearly interpolates both the inputs and the labels of pairs of examples. After decomposition, the empirical loss term, represented by cross-validation loss, can be accurately calculated, while the finite difference or vicinal Jensen gap serves as a regularization objective to control overfitting. These two objectives can be optimized using a weighted-sum objective function. However, different datasets with varying noise levels may require different regularization strengths. To address this, we propose a noise estimation strategy to adaptively adjust the regularization strength based on the estimated noise level in the dataset. This adaptive approach highlights a key benefit of vicinal risk decomposition: it allows flexible control over model complexity without imposing excessive regularization that could hinder model fitting.
        In summary, the overall goal of this paper is to enhance generalization in evolutionary feature construction for symbolic regression by decomposing vicinal risk into empirical loss and regularization terms. This decomposition enables effective overfitting control by flexibly balancing model accuracy and smoothness. The objectives/contributions of this paper are summarized as follows:
        \begin{itemize}
            \item A regularized GP is proposed to optimize both cross-validation loss and either the vicinal Jensen gap or the finite difference to control the overfitting in evolutionary feature construction algorithms~\footnote{Source Code: \url{https://github.com/hengzhe-zhang/EvolutionaryForest/blob/master/experiment/methods/VJM_GP.py}}. Additionally, a noise estimation strategy is proposed to adjust the regularization strength adaptively.
            \item An upper bound on perturbation-based vicinal risk minimization is derived, demonstrating that perturbation-based vicinal risk minimization can be decomposed into empirical risk and finite difference.
            \item An upper bound on mixup-based vicinal risk minimization is established, showing that mixup-based vicinal risk minimization can be decomposed into empirical risk and the vicinal Jensen gap.
            \item To alleviate the discrepancy between synthesized and real data, a manifold intrusion detection strategy is proposed, thereby avoiding the issue of Jensen gap minimization overly penalizing effective models.
        \end{itemize}
        The remainder of this paper is organized as follows. \cref{sec: Related Work} reviews existing techniques for overfitting control and evolutionary feature construction. \cref{sec: Vicinal Risk Decomposition} decomposes the vicinal risk into empirical loss and regularization terms, laying the foundation for the proposed algorithm. \cref{sec: Algorithm} describes the proposed algorithm in detail. \cref{sec: Experimental Settings} presents the experimental settings, followed by the experimental results in \cref{sec: Experimental Results} and additional analysis in \cref{sec: Further Analysis}. Finally, \cref{sec: Conclusion} summarizes the paper and suggests potential directions for future research.
        \section{Related Work}
        \label{sec: Related Work}
        \subsection{Vicinal Risk Minimization}
        \subsubsection{Vicinal Risk Minimization in Deep Learning}
        Vicinal risk minimization is a well-established technique that has achieved significant success in deep learning, particularly in computer vision tasks. Vicinal examples can be synthesized under various assumptions, such as randomly masking out a patch of an image (CutOut)~\cite{sajjadi2016regularization}, blending two images (MixUp)~\cite{zhang2018mixup}, or combining an image with its augmented version (AugMix)~\cite{hendrycks2019augmix}. However, most research on vicinal risk minimization has focused on image classification problems, while its application to tabular data and regression tasks remains limited~\cite{yao2022c}.
        \subsubsection{Vicinal Risk Minimization in GP}
        In the context of GP, VRM has primarily been applied to classification tasks~\cite{ni2015training}, under the assumption that the target class $y_i$ remains stable within the vicinity of training examples. Accordingly, VRM has been used to improve the generalization of GP-based classifiers~\cite{ni2015training} and evolutionary decision tree classifiers~\cite{cao2015use}. However, this assumption—that target labels $y_i$ remain unchanged with small perturbations in the training data—may not hold for regression tasks, where target values are continuous. Therefore, further investigation is required to adapt VRM for regression datasets.
        \subsection{Overfitting Control for GP}
        The issue of generalization in GP has gained significant attention in recent years because the goal of GP is not only to achieve good fitness but also to obtain models that perform well on unseen data~\cite{agapitos2019survey}. Early work on controlling overfitting in GP focused on reducing the model size. However, recent research indicates that generalization performance is more complex than merely achieving a small model~\cite{vanneschi2010measuring}. Based on these findings, researchers in GP have explored tools from theoretical machine learning, such as Tikhonov regularization~\cite{ni2014tikhonov}, VC-dimension~\cite{chen2018structural}, and Rademacher complexity~\cite{chen2020rademacher}, to control overfitting. However, modern feature learning techniques, such as deep neural networks, possess high VC dimensions yet generalize well, making traditional complexity measures fail to accurately predict the generalization performance of modern machine learning methods~\cite{zhang2021understanding}. Another type of overfitting control technique in machine learning is ensemble-based learning, either through homogeneous ensembles~\cite{EF-TEVC} or heterogeneous ensembles~\cite{zhang2023sr}, based on bias-variance theory~\cite{owen2020characterizing}. However, ensemble methods are not suitable in scenarios where interpretability is essential. Finally, several practical machine learning techniques, such as random sampling~\cite{gonccalves2013balancing}, early stopping~\cite{tuite2011early}, semi-supervised learning~\cite{silva2018semi}, and soft targets~\cite{vanneschi2021soft}, have shown effectiveness in GP, although the theoretical foundation of these methods is limited.
        \subsection{Evolutionary Feature Construction}
        Based on the evaluation method, evolutionary feature construction can be categorized into filter-based, wrapper-based, and embedded methods. The filter-based method evaluates features using general metrics, such as information gain~\cite{ma2023multi} or impurity~\cite{neshatian2012filter}. This method is fast and can construct features that generalize well across different classifiers. Wrapper-based methods evaluate features based on a specific learning algorithm and can achieve superior performance with that algorithm. For wrapper-based feature construction, multi-tree GP methods have been widely used and show superior performance, such as Multidimensional Multiclass GP with Multidimensional Populations (M3GP)~\cite{munoz2019evolving}, Feature Engineering Automation Tool (FEAT)~\cite{la2020learning}, and Gene-pool Optimal Mixing Evolutionary Algorithm for GP (GP-GOMEA)~\cite{virgolin2020explaining}. The embedded method integrates feature construction directly into the learning process, with symbolic regression~\cite{wang2023shapley} being a typical example. This paper focuses on wrapper-based feature construction due to its superior performance with a given learning algorithm.
        \section{Vicinal Risk Decomposition}
        \label{sec: Vicinal Risk Decomposition}
        In this section, we decompose vicinal risk minimization into loss and regularization terms. Based on this decomposition, we propose an empirical method for calculating the regularization terms in \cref{sec: Empirical Vicinal Risk Decomposition}, which are used to control overfitting in GP.
        \begin{theorem}
            Let $f: \mathcal{X} \rightarrow \mathcal{Y}$ be a machine learning model mapping from the input space $\mathcal{X}$ to the output space $\mathcal{Y}$. Consider the VRM framework where each input $x_i \in \mathcal{X}$ is perturbed by a noise vector $\epsilon$, resulting in a vicinal training instance with input $x_{\text{vic}} = x_i + \epsilon$ and the corresponding target $y_{\text{vic}}=y_i$. Under these conditions, the following bound holds:
            \begin{equation}
                (y_{\text{vic}} - f(x_{\text{vic}}))^2 \leq 2(y_i - f(x_i))^2 + 2(f(x_i) - f(x_i + \epsilon))^2.
                \label{eq: Perturbation VRM}
            \end{equation}
            \label{the: Perturbation VRM}
        \end{theorem}
        The proof is provided in supplementary material. The first term of the bound in \cref{eq: Perturbation VRM} is the traditional empirical loss, $(y_i - f(x_i))^2$. The second term represents the stability of the model under perturbations, specifically $(f(x_i) - f(x_i + \epsilon))^2$, which corresponds to the square of a finite difference value. The interpretation of this finite difference depends on $\epsilon$. If $\epsilon$ is sufficiently small, the second term, $(f(x_i) - f(x_i + \epsilon))^2$, approximates the square of the derivative.
        In practice, for VRM based on finite difference minimization, the objective or fitness value of a GP individual $\Phi$ with a linear model $F$ can be formulated as follows:
        \begin{equation}
            \text { minimize }
            \left\{
                \begin{aligned}
                    O_1(\Phi) &= \sum_{x,y \in (X,Y)} \left[F(\Phi(x)) - y\right]^2,\\
                    O_2(\Phi) &= \sum_{x \in X} \left[F(\Phi(x+\epsilon)) - F(\Phi(x))\right]^2.
                \end{aligned}
            \right.
            \label{eq: gaussian-based VRM}
        \end{equation}
        The perturbation $\epsilon$ is usually generated by Gaussian noise, but the objective function shown in \cref{eq: gaussian-based VRM} is applicable to any perturbation noise.
        In fact, the objective function $\left[F(\Phi(x+\epsilon))-F(\Phi(x))\right]^2$, which measures the difference between predictions on original data $F(\Phi(x))$ and corrupted data $F(\Phi(x+\epsilon))$, has been used in evolutionary feature construction~\cite{zhang2024bias} and symbolic regression~\cite{bakurov2024sharpness} for variance reduction~\cite{zhang2024bias} and sharpness minimization~\cite{bakurov2024sharpness}, respectively. While perturbation-based VRM methods are straightforward to implement, they have limitations. Specifically, most regression models respond continuously to input changes, making the regularization term $\left[F(\Phi(x+\epsilon))-F(\Phi(x))\right]^2$ somewhat restrictive. Therefore, in the following, we decompose mixup-based VRM, deriving a regularization term more suitable for evolutionary feature construction in regression tasks.
        \begin{theorem}
            Within the framework of mixup-based VRM,  for any pair of samples $(x_i, y_i)$ and $(x_j, y_j)$, where $(x_j, y_j)$ is in the neighborhood of $(x_i, y_i)$, and for a mixup coefficient $\lambda \in [0,1]$, the loss on the vicinal example defined by $x_{\text{vic}} = \lambda x_i + (1-\lambda) x_j$ and $y_{\text{vic}} = \lambda y_i + (1-\lambda) y_j$ satisfies the following inequality:
            \begin{equation}
                \begin{split}
                    (y_{\text{vic}} - f(x_{\text{vic}}))^2 &\leq 3(\lambda (y_i - f(x_i)))^2 \\
                    &\quad + 3((1-\lambda) (y_j - f(x_j)))^2 \\
                    &\quad + 3\left(\lambda f(x_i) + (1-\lambda) f(x_j) - f(x_{\text{vic}})\right)^2.
                \end{split}
                \label{eq: Mixup VRM Theorem}
            \end{equation}
            \label{the: VJM}
        \end{theorem}
        The proof is provided in supplementary material. \cref{the: VJM} shows that minimizing mixup-based VRM can be achieved by minimizing the vicinal Jensen gap around all training instances $x_i$ in the training data $X$. It is worth noting that the term ``vicinal'' here actually relates to how a sample $x_j$ is sampled from the neighborhood of sample $x_i$. If the neighborhood definition is infinitely large, the vicinal Jensen gap minimization will degenerate into minimizing the global Jensen gap of the function $f$.
        \cref{eq: Mixup VRM Theorem} consists of two parts. The first part is the mean squared loss $(\lambda (y_i - f(x_i)))^2 + ((1-\lambda) (y_j - f(x_j)))^2$, and the second part is $\left(\lambda f(x_i) + (1-\lambda) f(x_j) - f(x_{\text{vic}})\right)^2$. These two parts can form two objective functions, $O_1(\Phi)$ and $O_2(\Phi)$.
        Thus, for mixup-based VRM, the objective function is formulated as:
        \begin{equation}
            \text { minimize }
            \left\{
                \begin{aligned}
                    O_1(\Phi) &= \sum_{(x,y) \in (X,Y)} \left[F(\Phi(x)) - y\right]^2,\\
                    O_2(\Phi) &= \sum_{x_i, x_j \in X} \bigg[\lambda F(\Phi(x_i)) \\
                    &\quad + (1-\lambda) F(\Phi(x_j)) \\
                    &\quad - F( \Phi(\lambda x_i+(1-\lambda) x_j))\bigg]^2.
                \end{aligned}
            \right.
            \label{eq: mixup-based VRM}
        \end{equation}
        The traditional vicinal risk is $\mathcal{V}(f)=\frac{1}{n} \sum_{i=1}^n \int \mathcal{L} \left(f(x), y\right) d P_{x_i}(x)$. From the decomposition results, it is clear that the target $Y$ is eliminated from the regularization term $O_2(\Phi)$, making it a better regularization objective than the traditional vicinal risk $\mathcal{V}(f)$. This regularization term is only related to the smoothness of the evolved function itself, while the fitting accuracy is determined solely by the empirical loss $O_1(\Phi)$. It is worth noting that the regularization term is still data-dependent because $X$ remains involved, meaning the smoothness is not enforced across the entire feature space but only within regions where samples exist. This leads to a more realistic regularization for real-world applications, especially when the data lies on a low-dimensional manifold embedded in a high-dimensional space. Furthermore, smoothness is considered between nearby pairs of $x_i$ and $x_j$, rather than between all pairs of samples in the training set, further avoiding over-penalization. This approach better respects the local geometry of the data manifold rather than enforcing unnecessary global smoothness across irrelevant regions. Such a decomposition enables flexible control over the trade-off between the fitting accuracy and the smoothness of the evolved function.
        \begin{figure}[!tb]
            \centering
            \begin{subfigure}[b]{\columnwidth}
                \centering
                \includegraphics[width=\columnwidth]{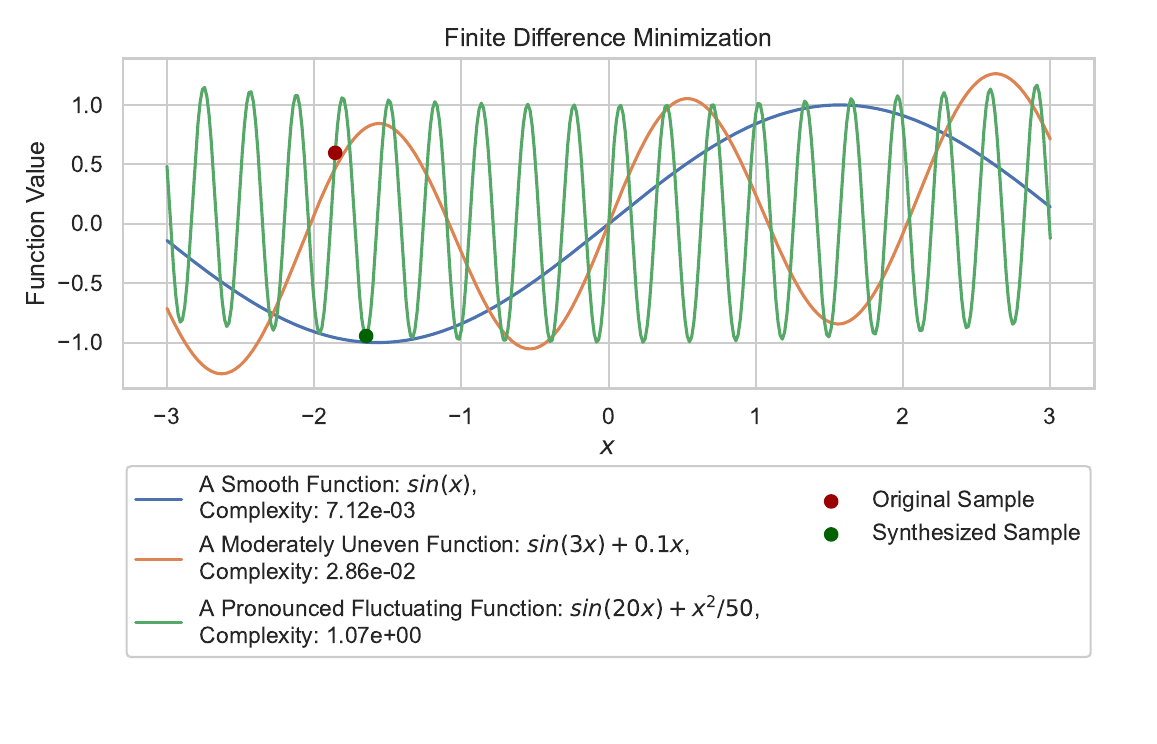}
                \vspace{-1cm}
                \caption{Gaussian noise-based VRM $\rightarrow$ Finite Difference}
                \label{fig: LC}
            \end{subfigure}
            \hfill
            \begin{subfigure}[b]{\columnwidth}
                \centering
                \includegraphics[width=\columnwidth]{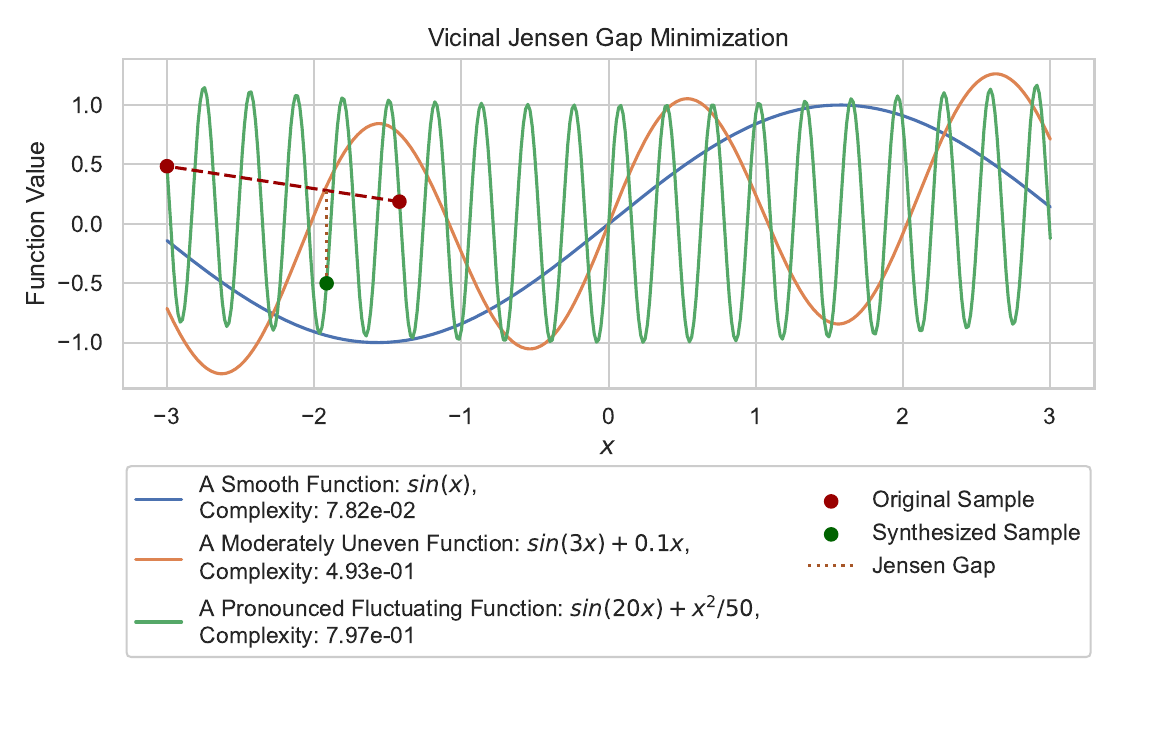}
                \vspace{-1cm}
                \caption{Mixup-based VRM $\rightarrow$ Vicinal Jensen Gap}
                \label{fig: LJ}
            \end{subfigure}
            \caption{Illustrative examples of Gaussian noise-based and mixup-based VRM. The predictions of the original and synthesized samples made by the fluctuating function are shown.}
            \vspace{-0.35cm}
        \end{figure}
        We illustrate the decomposition with intuitive examples for both Gaussian noise-based VRM and mixup-based VRM in \cref{fig: LC} and \cref{fig: LJ}, respectively. As depicted in \cref{fig: LC}, minimizing the loss on synthetic samples generated with Gaussian noise derives an overfitting control method that penalizes rapid changes in the function’s output in response to small input variations. Complexity in this case is measured by the finite difference $\left[F(\Phi(x+\epsilon)) - F(\Phi(x))\right]^2$, as defined in \cref{eq: gaussian-based VRM}. \cref{fig: LJ} illustrates that minimizing the loss on samples synthesized via the mixup method derives an overfitting control method that penalizes convexity and concavity, thereby promoting local linearity. In this case, complexity is measured by the vicinal Jensen gap $\bigg[\lambda F(\Phi(x_i)) + (1-\lambda) F(\Phi(x_j)) - F( \Phi(\lambda x_i+(1-\lambda) x_j))\bigg]^2$, as defined in \cref{eq: mixup-based VRM}.
        Here, $\lambda$ is sampled from a Beta distribution $\text{Beta}(\alpha, \alpha)$, which is widely used to sample the mixing coefficient in mixup methods~\cite{zhang2018mixup}.
        \section{The Proposed Algorithm}
        \label{sec: Algorithm}
        In this section, we first introduce the general framework for evolutionary feature construction that integrates the proposed vicinal Jensen gap or finite difference. Next, we propose an empirical method for calculating the regularization term, based on the decomposed regularization term described in \cref{sec: Vicinal Risk Decomposition}. Then, we propose a noise level estimation strategy to adaptively determine the regularization strength. Finally, we introduce a manifold intrusion detection mechanism to prevent the Jensen gap minimization from overly penalizing effective models.
        \subsection{Algorithm Framework}
        \label{sec: Algorithm Framework}
        In this paper, we propose a Vicinal Jensen Gap Minimization-based GP (VJM-GP) for feature construction. Alternatively, the regularization objective of the vicinal Jensen gap can be replaced with finite difference minimization, resulting in the Finite Difference Minimization-based GP (FDM-GP). Both variants share the same overall algorithmic framework. Our framework builds on the conventional structure of evolutionary feature construction while introducing novel fitness functions.
        As illustrated in \cref{fig: Workflow}, two key steps are integrated into the conventional evolutionary feature construction framework: first, synthesizing vicinal data at the beginning of the evolution process, and second, using the synthesized data during the evolution process to calculate the vicinal Jensen gap or finite difference.
        \begin{figure}[!tb]
            \centering
            \includegraphics[width=0.5\columnwidth, trim=15pt 15pt 15pt 15pt, clip]{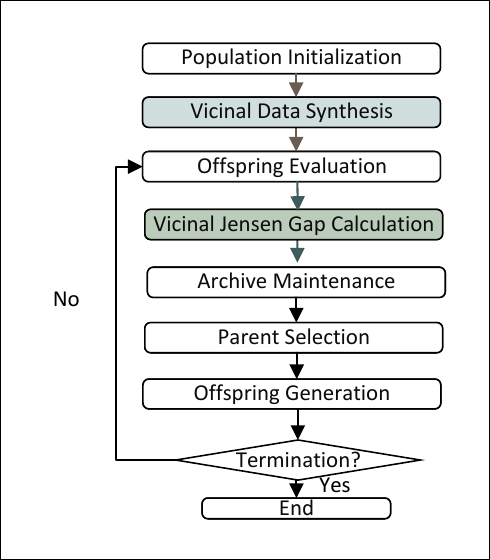}
            \caption{Workflow of Vicinal Jensen Gap Minimization-based GP.}
            \label{fig: Workflow}
        \end{figure}
        \begin{itemize}
            \item \textbf{Population Initialization:} In the initialization phase, a population $P$ is randomly generated by initializing a set of individuals. Each individual initially contains a single GP tree. However, during the evolutionary process, individuals can evolve to include multiple GP trees, each representing a constructed feature~\cite{munoz2019evolving}. The number of GP trees per individual is dynamically determined through tree addition and deletion operators applied during the mutation stage.
            \item \textbf{Parent Selection:} Parents are selected from the population using the lexicase selection operator~\cite{la2019probabilistic}. Lexicase selection operates by randomly generating filters to iteratively narrow down the population, thereby identifying promising parent individuals. For each filter corresponding to a randomly selected instance $k$, the minimum squared error $\min_{\Phi' \in P} \mathcal{L}_k(\Phi')$ is computed across all individuals $\Phi' \in P$, along with the median absolute deviation $\epsilon_{k}$ of GP individuals for that instance. Individuals that satisfy $\mathcal{L}_k(\Phi) \leq \min_{\Phi' \in P} \mathcal{L}_k(\Phi') + \epsilon_{k}$ are retained, while others are eliminated. This process is repeated with multiple filters until only one individual remains, which is then selected as a parent. To select multiple parents, the lexicase selection operator is invoked multiple times.
            \item \textbf{Offspring Evaluation:} Offspring are evaluated using cross-validation and vicinal risk estimation to obtain objective values $O_1(\Phi)$ and $O_2(\Phi)$, respectively. Typically, cross-validation is performed by partitioning the data into multiple folds and executing linear regression on each fold. However, for linear regression, leave-one-out cross-validation (LOOCV) can be computed efficiently~\cite{hastie2009elements}. Instead of solving $n$ separate least-squares optimization problems, the LOOCV error vector $\boldsymbol{e}$ can be calculated as \cref{eq: LOOCV}:
            \begin{equation}
                \boldsymbol{e} = \left( \frac{\mathbf{Y} - \mathbf{X} \boldsymbol{w}}{1 - \boldsymbol{h}} \right)^2,
                \label{eq: LOOCV}
            \end{equation}
            where $\boldsymbol{w}$ denotes the optimal weights, and $\boldsymbol{h}$ represents the leverage values. The optimal weights $\boldsymbol{w}$ for a regularization coefficient $\alpha$ are computed as:
            \begin{equation}
                \boldsymbol{w} = \left( \mathbf{X}^\top \mathbf{X} + \alpha \mathbf{I} \right)^{-1} \mathbf{X}^\top \mathbf{Y}.
                \label{eq: weight}
            \end{equation}
            The leverage values $\boldsymbol{h}$ are calculated as:
            \begin{equation}
                \boldsymbol{h} = \text{diag}\left(\mathbf{X} \left( \mathbf{X}^\top \mathbf{X} + \alpha \mathbf{I} \right)^{-1} \mathbf{X}^\top\right).
                \label{eq: leverage}
            \end{equation}
            Intuitively, the leverage $h_{ii}$ measures the influence of each data point $x_i$ on the model's weights, with larger values indicating greater impact. Using \cref{eq: LOOCV}, a vector of leave-one-out cross-validation errors is computed, which contains $n$ values representing the loss for each of the $n$ training instances. In addition to cross-validation loss, regularization terms such as finite difference or vicinal Jensen gap are calculated based on the method introduced in \cref{sec: Empirical Vicinal Risk Decomposition}.
            \item \textbf{Offspring Generation:} Offspring are generated from the selected parents using genetic operators including random subtree crossover, mutation, tree addition, and tree deletion~\cite{munoz2019evolving}. Random subtree crossover exchanges subtrees between two parent GP trees, facilitating the combination of building blocks from different individuals. Random subtree mutation introduces new genetic material by modifying parts of a GP tree. The tree addition and deletion operators modify the number of trees in each individual by either adding a GP tree using the ramped-half-and-half method or randomly removing a GP tree, dynamically adjusting the number of constructed features.
            \item \textbf{Environmental/Survival Selection:} The environmental selection process selects a population of $N$ individuals from the combined pool of parents and offspring using non-dominated sorting with crowding distance~\cite{deb2002fast}. This approach ranks individuals based on Pareto dominance and maintains diversity within the population by considering crowding distance. The top-$N$ individuals are chosen to survive to the next generation. This selection mechanism is essential because lexicase selection only considers the semantics of individuals, neglecting the regularization term. By incorporating non-dominated sorting, the environmental selection operator ensures that both objectives—cross-validation loss and regularization—are simultaneously optimized, promoting better generalization performance.
            \item \textbf{Archive Maintenance:} An external archive is maintained to store the historically best individuals for making predictions on unseen data. Throughout the optimization process, a Pareto set of two objectives is obtained. However, for making predictions on unseen data, a scalarized objective function is required to select a final model. Specifically, the best model is determined by the minimum sum of two objectives, i.e.,
            \begin{equation}
                O_1(x) + \tau O_2(x),
                \label{eq: Objective}
            \end{equation}
            where $O_1(x)$ is the cross-validation loss, $O_2(x)$ is the regularization term, i.e., the vicinal Jensen gap, and $\tau$ is the balancing coefficient. By default, $\tau$ is set to 1. However, for highly noisy datasets, it might be better to set $\tau$ to a larger value to reduce fitting noise, with the noise estimation strategy shown in \cref{sec: Noise Estimation Strategy}.
        \end{itemize}
        The processes of parent selection, offspring generation, offspring evaluation, environmental selection, and archive maintenance are iteratively repeated until the desired number of iterations is reached. Finally, the best model in the archive, as determined by \cref{eq: Objective}, is used to make predictions on unseen data.
        \subsection{Empirical Vicinal Risk Decomposition}
        \label{sec: Empirical Vicinal Risk Decomposition}
        At the beginning of empirical vicinal risk estimation, constructed features $\Phi(X)$ are first derived from the original features $X$. A linear regression model $F$ is then trained using these constructed features $\Phi(X)$ as input features. Once trained, the linear model $F$ remains fixed. The vicinal risk can then be estimated based on \cref{eq: gaussian-based VRM} or \cref{eq: mixup-based VRM}. To ensure robustness, vicinal risk is estimated over $K$ iterations, where $K$ is a hyperparameter. The process of empirical vicinal risk estimation includes four key steps, as outlined in \cref{alg: Vicinal Risk Estimation}.
        \begin{itemize}
            \item \textbf{Vicinal Data Synthesis (Line 4):} In this stage, vicinal data $X_{vic}$ is synthesized either by adding Gaussian noise or using the mixup method, depending on whether the VRM focuses on finite difference or the Jensen gap.
            \begin{itemize}
                \item \textbf{Finite Difference:} To optimize the finite difference, vicinal data is synthesized by adding Gaussian noise, as specified in \cref{eq: Gaussian VRM}:
                \begin{equation}
                    x_{vic}=x_i+\epsilon, \quad \epsilon \sim \mathcal{N}(0, \sigma^2).
                    \label{eq: Gaussian VRM}
                \end{equation}
                \item \textbf{Vicinal Jensen Gap:}
                To optimize the vicinal Jensen gap, new samples are synthesized using the mixup method, as defined in \cref{eq: Mixup VRM}:
                \begin{equation}
                    x_{vic}=\lambda x_i + (1-\lambda) x_j, \quad \lambda \sim \text{Beta}(\alpha, \beta),
                    \label{eq: Mixup VRM}
                \end{equation}
                where $\text{Beta}(\alpha, \beta)$ represents the Beta distribution with control parameters $\alpha$ and $\beta$. In this paper, both $\alpha$ and $\beta$ are set to 10, i.e., $\alpha = \beta = 10$, as setting them to the same value is common practice in mixup~\cite{zhang2018mixup}. A sensitivity analysis of $\alpha$ is provided in \cref{sec: Alpha} of the supplementary material, demonstrating that the method maintains stable performance across a range of $\alpha$ values.
                A key assumption of mixup is that vicinal samples are synthesized from nearby data points. Therefore, for each sample $x_i, y_i$, only samples $x_j, y_j$ that are similar to $x_i, y_i$ have a high probability of being mixed. To achieve this, we first define the distance kernel between two samples, as shown in \cref{eq: Kernel}:
                \begin{equation}
                    K(y_i, y_j) = e^{-\gamma \lVert y_i - y_j \rVert^2},
                    \label{eq: Kernel}
                \end{equation}
                where $\gamma$ is a spread parameter. Here, the distance is defined based on the distance between labels. Defining the distance based on either $y$ alone or both $x$ and $y$ has respective advantages, with the pros and cons discussed in \Cref{sec: Kernel Analysis} of the supplementary material.
                The kernel distance is normalized to compute the probability $P(y_j | y_i)$, as shown in \Cref{eq: Normalized Kernel}:
                \begin{equation}
                    P(y_j | y_i) = \frac{K(y_j, y_i)}{\sum_{y_k \in Y \setminus \{y_i\}} K(y_k, y_i)},
                    \label{eq: Normalized Kernel}
                \end{equation}
                where $y_k \in Y \setminus \{y_i\}$ ensures that the probability of sampling $y_i$ when the current instance is $y_i$ is explicitly set to zero to avoid redundant regularization.
                For each training sample $x_i$, $K$ paired samples are sampled according to the probability distribution $P(y_j | y_i)$.
            \end{itemize}
            For each random seed in $K$ iterations, the vicinal data varies. To ensure efficient utilization of computational resources and fair comparisons of vicinal risk, the synthesized vicinal data is cached and reused across different individuals.
            \item \textbf{Vicinal Feature Construction (Line 6):} Features are constructed for the synthesized vicinal data $X_{vic}$. These features are then input into the fixed model $F$ to obtain predictions $\hat{Y}$.
            \item \textbf{Vicinal Jensen Gap/Finite Difference Calculation (Lines 7-8):} The predicted results $\hat{Y}$ are compared against the target labels.
            \begin{itemize}
                \item \textbf{Gaussian Noise:} For samples synthesized with Gaussian noise, the target labels $y_{vic}$ remain the same as the original target $y_i$.
                \item \textbf{Mixup:}  For samples synthesized using mixup, the target label $y_{vic}$ is defined as in \cref{eq: Mixup Target}, based on \cref{eq: mixup-based VRM}:
                \begin{equation}
                    y_{vic} = \lambda \cdot F(\Phi(x_i)) + (1 - \lambda) \cdot F(\Phi(x_j)).
                    \label{eq: Mixup Target}
                \end{equation}
            \end{itemize}
            In both cases, we compute the worst-case squared error between the prediction $\hat{y}_i$ and the target label $y_{vic}$ for each instance over $K$ iterations. The worst-case squared error is used instead of the average squared error to encourage the development of a model that is more robust to unseen data across different scenarios.
            Specifically, for each $x_i$, the vicinal Jensen gap or finite difference is computed as:
            \begin{equation}
                \mathcal{V}_i = \max_{k \in \{1, \dots, K\}} (\hat{y}_i^k - y_{vic}^k)^2,
                \label{eq: Pessimistic Risk}
            \end{equation}
            where $\hat{y}_i^k$ and $y_{vic}^k$ represent the prediction and target label for the $k$-th iteration, respectively. The final regularization term $\mathcal{V}$ is the average of $\mathcal{V}_i$ over all training instances, as detailed in \cref{alg: Vicinal Risk Estimation}.
        \end{itemize}
        \begin{algorithm}[!tb]
            \caption{Vicinal Jensen Gap/Finite Difference Minimization }
            \label{alg: Vicinal Risk Estimation}
            \begin{algorithmic}[1]
                \Require GP Tree $\Phi$, Input $X$, Target Outputs $Y$, Linear Model $F$, Number of Iterations $K$
                \State Initialize Regularization Term: $\mathcal{V} \gets 0$
                \State $\Phi(X) \gets $ Feature Construction ($X, \Phi$)
                \For{$k = 1, \ldots, K$}
                    \State $X_{vic}, Y_{vic} \gets$ Data Synthesis ($X, Y, k$) \Comment{Cached}
                    \State $\Phi(X_{vic}) \gets $ Feature Construction ($X_{vic}, \Phi$)
                    \State $\hat{Y} \gets$  Prediction($F, \Phi(X_{vic})$)
                    \For{$i = 1, \ldots, N$}
                        \State $\mathcal{V}_i \gets \max(\mathcal{V}_i, (\hat{y}_i^k - y_{vic}^k)^2)$
                    \EndFor
                \EndFor
                \Ensure Regularization Term $\mathcal{V} = \frac{1}{N} \sum_{i=1}^{N} \mathcal{V}_i$
            \end{algorithmic}
        \end{algorithm}
        \subsection{Noise Estimation Strategy}
        \label{sec: Noise Estimation Strategy}
        \begin{figure}[!tb]
            \centering
            \includegraphics[width=\linewidth]{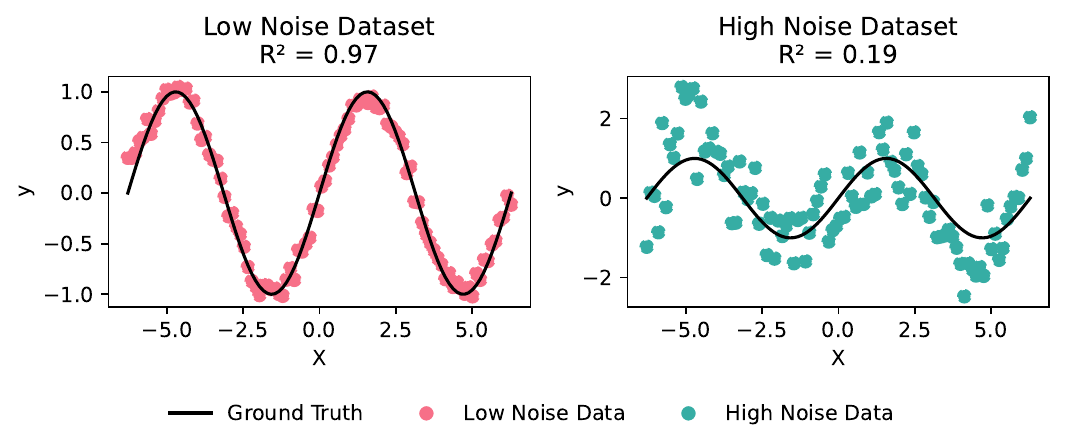}
            \caption{Estimating Noise Level with Extremely Randomized Trees.}
            \label{fig: noise estimation}
        \end{figure}
        There are various strategies for adjusting the weight $\tau$ in \cref{eq: Objective} based on estimated noise level or estimated complexity of datasets. In this paper, we propose adjusting $\tau$ according to the noise level in the dataset. This approach is supported by observations in \cref{fig: noise estimation}, which illustrate that the cross-validation predictions of the same sinusoidal function vary with different noise levels.
        Specifically, when the noise level is low, Extremely Randomized Trees (Extra Trees)~\cite{geurts2006extremely} can make accurate predictions. However, as the noise level increases, these trees tend to produce less accurate predictions.
        Consequently, we use the five-fold cross-validation $R^2$ score of training data with Extra Trees as a metric to estimate the noise level in the datasets. Based on the estimated noise level $R^2$, the weighting coefficient $\tau$ is defined as follows:
        \begin{equation}
            \tau =
            \begin{cases}
                1 & \text{if } R^2 \geq 0.5, \\
                10 & \text{if } R^2 < 0.5.
            \end{cases}
        \end{equation}
        \subsection{Manifold Intrusion Detection}
        \label{sec: Manifold Intrusion Detection}
        The vicinal Jensen gap minimization encourages local linearity between two selected points $x_i$ and $x_j$. However, Jensen gap minimization might conflict with the true underlying function if the relationship between $x_i$ and $x_j$ is highly nonlinear, leading to a phenomenon known as manifold intrusion~\cite{guo2019mixup}. An example of manifold intrusion is illustrated in \cref{fig: manifold_intrusion}, where the label of a synthesized sample deviates significantly from the true function due to incorrect assumptions about linearity.
        To mitigate manifold intrusion in synthesized vicinal data, we propose a manifold intrusion detection strategy to discard intrusion points. In this strategy, we predict the label for synthesized data points $x_{\text{vic}}$ using Extra Trees $T$ trained on the original training data $(X,Y)$, denoted as $y_{\text{tree}} = T(x_{\text{vic}})$, as indicated by the square in \cref{fig: manifold_intrusion}. Assuming $\Phi(x_i)< \Phi(x_j)$, the lower and upper bounds for the synthesized instance are defined as:
        \begin{equation}
            \text{Lower Bound} = (\lambda - \mu) \cdot F(\Phi(x_i)) + (1 - (\lambda - \mu)) \cdot F(\Phi(x_j)),
        \end{equation}
        \begin{equation}
            \text{Upper Bound} = (\lambda + \mu) \cdot F(\Phi(x_i)) + (1 - (\lambda + \mu)) \cdot F(\Phi(x_j)),
        \end{equation}
        where $\mu$ is a hyperparameter representing the margin. The value of $\mu$ depends on the confidence in the accuracy of the reference model, which in this paper is the Extra Trees regressor. If the reference model is expected to be accurate, $\mu$ is set to a small value. Conversely, if the reference model's reliability is questionable, $\mu$ is set to a large value. In this paper, we use the cross-validation $R^2$ scores on the training set, as calculated in \cref{sec: Noise Estimation Strategy}, to estimate the reliability of the reference model. If the model is deemed reliable, $\mu$ is set to 0.05, ensuring that the interval $\text{Upper Bound} - \text{Lower Bound}$ equals $0.1 \cdot \Big[F(\Phi(x_i)) - F(\Phi(x_j))\Big]$. Otherwise, $\mu$ is set to $\infty$, effectively disabling manifold intrusion detection. If $\Phi(x_i)>\Phi(x_j)$, the lower and upper bounds are swapped accordingly.
        Based on the lower and upper bounds, if the predictions from Extra Trees $T$ significantly deviate from the synthesized data values, it suggests that the region is difficult to synthesize accurately, and the sample should be discarded. Formally, data points are discarded if either of the following conditions holds:
        \begin{equation}
            y_{\text{tree}} < \text{Lower Bound} \quad \text{or} \quad y_{\text{tree}} > \text{Upper Bound}.
            \label{eq: Manifold Intrusion}
        \end{equation}
        These conditions indicate that minimizing the Jensen gap in that region is inconsistent with the true data distribution, suggesting the synthesized point should be excluded. Data points that meet these conditions are discarded, and new samples are generated. To prevent infinite regeneration, after every ten unsuccessful attempts, the $\alpha$ parameter in the Beta distribution in \Cref{eq: Mixup VRM} is multiplied by 10 while keeping $\beta$ unchanged, thereby increasing the weight of $x_i$ in the synthetic data. The maximum number of regeneration attempts is capped at 100. If this limit is reached, the data is left as is.
        \begin{figure}[!tb]
            \centering
            \includegraphics[width=\linewidth]{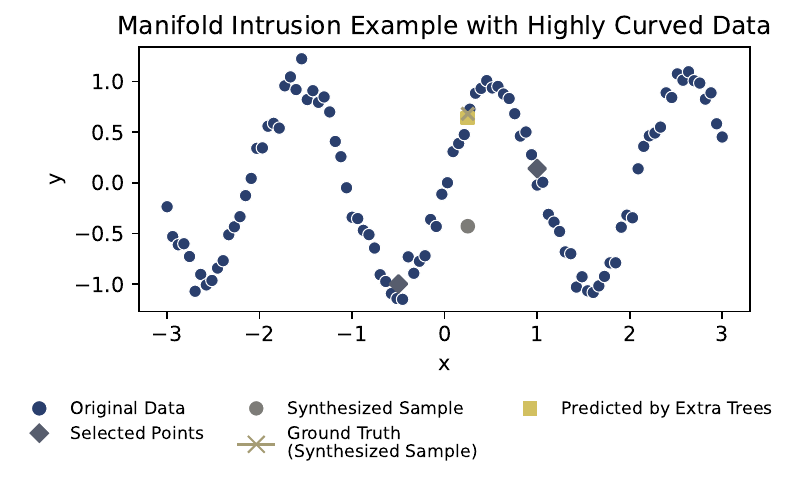}
            \caption{An example of manifold intrusion, where the ground truth and the prediction by Extra Trees overlap, but the synthesized sample deviates significantly from these two.}
            \label{fig: manifold_intrusion}
        \end{figure}
        \section{Experimental Settings}
        \label{sec: Experimental Settings}
        \subsection{Datasets}
        The regression datasets used in this paper are sourced from 120 black-box datasets in the Penn Machine Learning Benchmark (PMLB)~\cite{olson2017pmlb}. We focus exclusively on real-world datasets due to their inherent complexity. After excluding datasets synthesized by Friedman, a total of 58 real-world datasets are used in the experiments.
        \subsection{Parameter Settings}
        The parameter settings are detailed in \cref{tab: parameter settings}. Following tradition, a combination of a high crossover rate and a low mutation rate is used to facilitate the leveraging of building blocks. To prevent division-by-zero errors, the analytical quotient (AQ)~\cite{ni2012use} is used to replace the conventional division operator, which is defined as $AQ=\frac{a}{\sqrt{1 + b^2}}$. For the minimum and maximum operators, two arguments are allowed. The sine and cosine functions are defined as $\sin(\pi \times x)$ and $\cos(\pi \times x)$, respectively. Ephemeral constants are uniformly sampled from the range $[-1, 1]$. The archive elitism mechanism is used only in standard GP, as the non-dominated sorting~\cite{deb2002fast} in the proposed method inherently preserves elite individuals. For finite difference regularization, the standard deviation of the Gaussian noise, $\sigma$, is set to 0.5. For vicinal Jensen gap regularization, the spread parameter $\gamma$ is set to 0.5.
        \begin{table}[!tb]
            \centering
            \caption{Parameter settings for VJM-GP.}
            \begin{tabular}{cc}
                \toprule
                \textbf{Parameter}              & \textbf{Value}                 \\
                \midrule
                Maximal Population Size         & 200                            \\
                Number of Generations           & 100                            \\
                Crossover and Mutation Rates    & 0.9 and 0.1                    \\
                Tree Addition Rate              & 0.5                            \\
                Tree Deletion Rate              & 0.5                            \\
                Initial Tree Depth              & 0-3                            \\
                Maximum Tree Depth              & 10                             \\
                Initial Number of Trees         & 1                              \\
                Maximum Number of Trees         & 10                             \\
                Elitism (Number of Individuals) & 1                              \\
                Iterations of Estimation ($K$)  & 10                             \\
                Functions                       & \makecell{+, -, *, AQ, Square,\\ Log, Sqrt, Max, Min, \\ Sin, Cos, Abs, Negative}                       \\
                \bottomrule
            \end{tabular}
            \label{tab: parameter settings}
        \end{table}
        \subsection{Evaluation Protocol}
        \label{sec: Evaluation Protocol}
        To ensure a reliable comparison, each method performs 30 independent runs on each dataset. In each run, 100 samples are randomly chosen as training data~\cite{nicolau2021choosing}, and the remaining samples are used for testing. For smaller datasets where this split results in fewer than 100 test samples, a 50\%-50\% split is applied to ensure a sufficient number of test samples. A target encoder~\cite{micci2001preprocessing} is used to transform categorical input variables into numerical variables, as mixup requires numerical data. For the evaluation metric, $R^2$ is employed to eliminate magnitude differences between different datasets. During prediction, the maximum and minimum label values in the training data are recorded, and the predictions are clipped to these values to prevent excessively deviating predictions. The $R^2$ score is defined as $R^2 = 1 - \frac{\sum_{i} (y_i - \hat{y}_i)^2}{\sum_{i} (y_i - \bar{y})^2}$, where $\hat{y}_i$ represents the predictions, $y_i$ denotes the ground truth, and $\bar{y}$ is the mean of the ground truth values.
        After obtaining test scores, a Wilcoxon signed-rank test with a significance level of 0.01 is used to determine whether one method is statistically better, similar, or worse than another method.
        \subsection{Baseline Algorithms}
        For the baseline algorithms, we compare the proposed VJM-GP with 8 baseline methods, including standard GP and seven GP methods with various complexity measures:
        \begin{itemize}
            \item Pessimistic Vicinal Risk Minimization (P-VRM)~\cite{zhang2024p}: P-VRM implicitly optimizes the vicinal Jensen gap as shown in \cref{the: VJM}. The objective functions of P-VRM are:
            \begin{equation}
                \text { minimize }
                \left\{
                    \begin{aligned}
                        O_1(\Phi)&=\sum_{x \in X} \left[F(\Phi(x))- Y\right]^2,\\
                        O_2(\Phi)&=\sum_{x_i, x_j \in X} \bigg[\lambda y_i + (1-\lambda) y_j \\
                        &- F(\Phi(\lambda x_i+ (1-\lambda) x_j))\bigg]^2.
                    \end{aligned}
                \right.
                \label{eq: VRM}
            \end{equation}
            The process by which P-VRM generates vicinal samples follows the same procedure outlined in \cref{sec: Empirical Vicinal Risk Decomposition}. The primary distinction between P-VRM and VJM is that P-VRM implicitly incorporates the training loss in its second objective.
            \item Parsimony Pressure (PP)~\cite{de2023alleviating}: PP controls model size to enhance generalization performance. Specifically, it optimizes the number of nodes across all GP trees in an individual $\Phi$.
            \item Tikhonov Regularization (TK)~\cite{ni2014tikhonov}: TK regularizes overly large predictions using the regularization term $||f(X)||$, where $f$ represents the combination of GP trees and the linear model for making predictions. This approach penalizes excessively large predictions to ensure stability on unseen data.
            \item Grand Complexity (GC)~\cite{ni2014tikhonov}: GC integrates PP and TK to control overfitting, treating the dominance relationship between these objectives as a single optimization objective .
            \item Rademacher Complexity (RC)~\cite{chen2020rademacher}: RC measures the capacity of a model to fit the training data with random labels. It is formally defined as:
            \begin{equation}
                \operatorname{R}_n(\mathcal{L})=\mathbb{E}\left[\sup _{l \in \mathcal{L}} \frac{1}{n} \sum_{i=1}^n \sigma_i l\left(x_i,y_i\right)\right],
            \end{equation}
            where $\sigma_i$ is a Rademacher random variable taking values in $\{-1,1\}$ with equal probability.
            \item Weighted Maximal Information Coefficient between Residuals and Variables (WCRV)~\cite{chen2020improving}: WCRV consists of two terms. For important features $x^k$ in the GP tree that have a large correlation with the target compared to the median correlation $mv$ among all features, i.e., $\operatorname{MIC}_{x^k, Y} \geq mv$, the first term $\operatorname{MIC}_{x^k, R}$ penalizes the correlation between high-importance features and the residuals, weighted by feature importance $\operatorname{MIC}_{x^k, Y}$, thus avoiding regularity in the residuals. For unimportant features in the GP tree, i.e., $\operatorname{MIC}_{x^k, Y}<mv$, the second term encourages the GP to select a minimal number of low-relevance features, $1-\operatorname{MIC}_{x^k, Y}$. Based on these criteria, WCRV is defined as follows:
            \begin{equation}
                \begin{split}
                    \operatorname{WCRV}(\Phi)=&\sum_{\operatorname{MIC}_{x^k, Y} \geq mv} \operatorname{MIC}_{x^k, Y}\times \operatorname{MIC}_{x^k, R} \\
                    &+\sum_{\operatorname{MIC}_{x^k, Y}<mv}\left(1-\operatorname{MIC}_{x^k, Y}\right).
                \end{split}
            \end{equation}
            Ideally, both terms should be low for a low-complexity model, and thus WCRV is minimized to encourage a low-complexity model.
            \item Correlation between Input and Output Distances (IODC)~\cite{vanneschi2021soft}: IODC encourages a linear relationship between input correlation $\mathrm{I}$ and output correlation $\mathrm{O}$, denoted by $\operatorname{Cov}(\mathrm{I}, \mathrm{O})$. Formally, IODC is defined as:
            \begin{equation}
                \operatorname{IODC}(\Phi)=\frac{\operatorname{Cov}(\mathrm{I}, \mathrm{O})}{\sigma_{\mathrm{I}} \sigma_{\mathrm{O}}},
            \end{equation}
            where $\sigma_{\mathrm{I}}$ and $\sigma_{\mathrm{O}}$ are the standard deviations of $\mathrm{I}$ and $\mathrm{O}$, respectively. The key idea here is that a low-complexity model should exhibit a high correlation between input and output distances. Therefore, IODC is maximized to promote low model complexity.
            \item Standard GP: Standard GP uses the cross-validation loss as the optimization objective.
        \end{itemize}
        To ensure a fair comparison, all baseline algorithms are implemented within the same GP framework as the proposed VJM-GP, introduced in \cref{sec: Algorithm Framework}. To balance training performance with model complexity, the minimum Manhattan distance (MMD)-based knee point selection method~\cite{chiu2016minimum} is used to select individuals from the Pareto front of solutions. Specifically, the individual with the minimal sum of the two objectives, $\mathcal{O}_1 + \mathcal{O}_2$, is selected as the final model. Here, $\mathcal{O}_1$ and $\mathcal{O}_2$ represent the normalized objective values, as complexity measures are on a different scale compared to mean square error for the baseline methods. For P-VRM, the final model is chosen based on the minimal vicinal risk among the models on the Pareto front, rather than using the knee point. This is because the second objective in P-VRM encompasses both empirical loss (i.e., the first objective) and the regularization term, as detailed in \cref{eq: Perturbation VRM}. Using the VRM objective alone is a standard practice in deep learning~\cite{zhang2018mixup}.
        \section{Experimental Results}
        \label{sec: Experimental Results}
        In this section, we first compare the vicinal Jensen gap (VJM) to finite difference (FDM) as regularizers. Given the superiority of VJM, we then focus on comparing VJM with seven alternative complexity measures in terms of test $R^2$ scores, training $R^2$ scores, model sizes, and training time within the same evolutionary feature construction framework. Finally, we compare VJM-GP with 15 symbolic regression and machine learning algorithms to demonstrate the effectiveness of evolutionary feature construction with VJM.
        \subsection{Vicinal Jensen Gap versus Finite Difference}
        \def\JSLCB{15}
        \def\JSLCW{5}
        \def\JSLCTB{7}
        In \cref{sec: Empirical Vicinal Risk Decomposition}, mixup-based VRM is decomposed into training loss and vicinal Jensen gap, while perturbation-based VRM is decomposed into training loss and finite difference. A comparison of $R^2$ scores for using the Jensen gap and finite difference on the training and test sets is presented in \cref{fig: Training R2 LC} and \cref{fig: Test R2 LC}, respectively. Experimental results demonstrate that regularizing GP with the Jensen gap generally yields superior performance compared to regularizing with finite difference. As illustrated in \cref{fig: Test R2 LC}, regularization with Jensen gap outperforms regularization with finite difference on \JSLCB{} datasets and performs worse on only \JSLCW{} datasets regarding test $R^2$ scores. This indicates that the Jensen gap is a more effective regularizer for controlling overfitting. Consequently, in the following sections, we focus on optimizing the vicinal Jensen gap.
        \begin{figure}[!tb]
            \begin{subfigure}{0.45\linewidth}
                \centering
                \includegraphics[width=\linewidth]{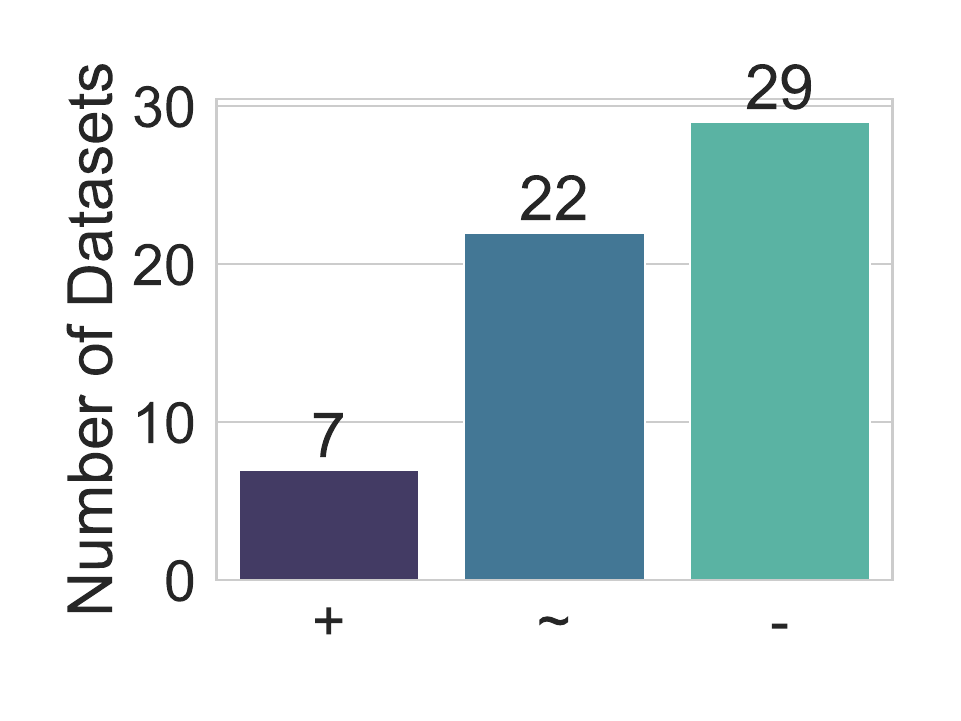}
                \caption{Training $R^2$ scores.}
                \label{fig: Training R2 LC}
            \end{subfigure}
            \begin{subfigure}{0.45\linewidth}
                \centering
                \includegraphics[width=\linewidth]{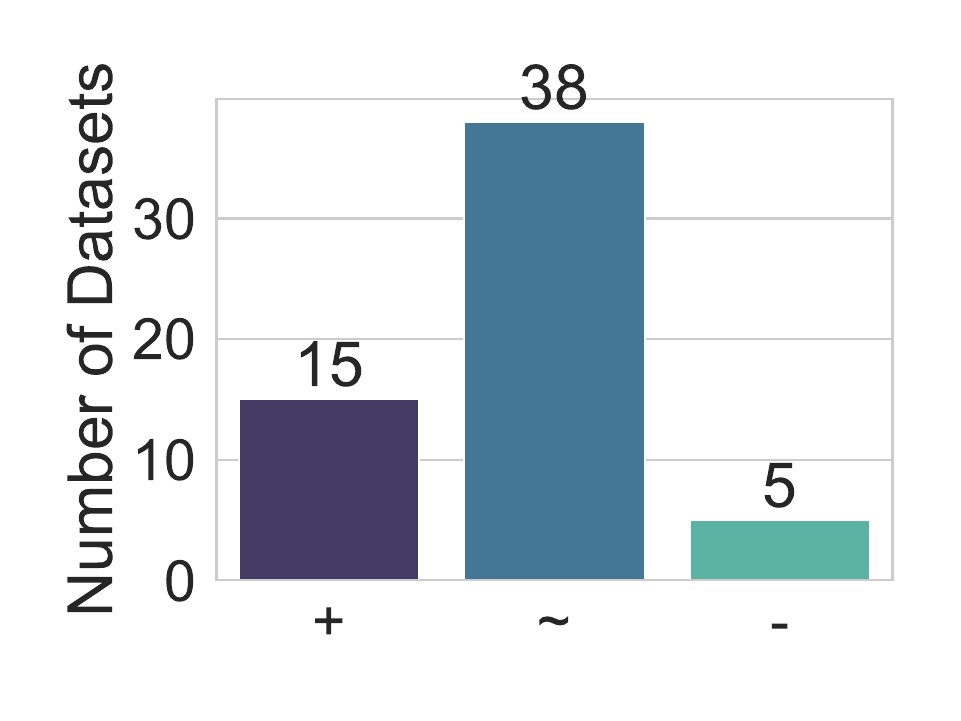}
                \caption{Test $R^2$ scores.}
                \label{fig: Test R2 LC}
            \end{subfigure}
            \caption{Comparison between optimizing vicinal Jensen gap versus finite difference. (``+",``$\sim$", and ``-" represent the number of datasets where optimizing vicinal Jensen gap performs better, similar, or worse, respectively, compared to optimizing finite difference.)}
        \end{figure}
        \subsection{Comparisons on Test $R^2$ Scores}
        \label{sec: Comparison}
        \begin{table*}[!tb]
            \centering
            \caption{Statistical comparison of \textbf{test $R^2$ scores} when optimizing various model complexity measures. (``+",``$\sim$", and ``-" indicate using the method in a row is statistically better than, similar to or worse than using the method in a column.)}
            \label{tab: Test R2}
            \resizebox{\textwidth}{!}{
                \begin{tabular}{ccccccccc}
                    \toprule
                    & \textbf{P-VRM}          & \textbf{PP}             & \textbf{RC}             & \textbf{GC}             & \textbf{IODC}            & \textbf{TK}              & \textbf{WCRV}& \textbf{Standard GP}\\
                    \midrule
                    \textbf{VJM}   & 22(+)/30($\sim$)/6({-}) & 29(+)/27($\sim$)/2({-}) & 51(+)/6($\sim$)/1({-})& 39(+)/18($\sim$)/1({-})& 39(+)/18($\sim$)/1({-})& 45(+)/13($\sim$)/0({-})& 43(+)/11($\sim$)/4({-})& 36(+)/17($\sim$)/5({-})\\
                    \textbf{P-VRM} & ---                     & 26(+)/26($\sim$)/6({-}) & 49(+)/7($\sim$)/2({-})  & 30(+)/27($\sim$)/1({-})& 33(+)/23($\sim$)/2({-})& 47(+)/10($\sim$)/1({-})& 39(+)/17($\sim$)/2({-})& 33(+)/18($\sim$)/7({-})\\
                    \textbf{PP}    & ---                     & ---                     & 36(+)/16($\sim$)/6({-}) & 19(+)/34($\sim$)/5({-}) & 22(+)/26($\sim$)/10({-})& 30(+)/27($\sim$)/1({-})& 24(+)/28($\sim$)/6({-})& 23(+)/26($\sim$)/9({-})\\
                    \textbf{RC}    & ---                     & ---                     & ---                     & 4(+)/16($\sim$)/38({-}) & 3(+)/31($\sim$)/24({-})  & 6(+)/25($\sim$)/27({-})& 8(+)/23($\sim$)/27({-})& 13(+)/14($\sim$)/31({-})\\
                    \textbf{GC}    & ---                     & ---                     & ---                     & ---                     & 17(+)/33($\sim$)/8({-})  & 23(+)/33($\sim$)/2({-})  & 15(+)/39($\sim$)/4({-})& 19(+)/23($\sim$)/16({-})\\
                    \textbf{IODC}  & ---                     & ---                     & ---                     & ---                     & ---                      & 20(+)/24($\sim$)/14({-}) & 15(+)/33($\sim$)/10({-})& 20(+)/15($\sim$)/23({-})\\
                    \textbf{TK}    & ---                     & ---                     & ---                     & ---                     & ---                      & ---                      & 9(+)/30($\sim$)/19({-})  & 9(+)/22($\sim$)/27({-})  \\
                    \textbf{WCRV}  & ---                     & ---                     & ---                     & ---                     & ---                      & ---                      & ---                      & 14(+)/25($\sim$)/19({-}) \\
                    \bottomrule
                \end{tabular}
            }
        \end{table*}
        \subsubsection{General Analysis}
        \def\VJMSTDB{36}
        \def\VJMSTDW{5}
        The comparison of test $R^2$ scores is presented in \cref{tab: Test R2}. In summary, different methods for controlling overfitting exhibit varying levels of effectiveness, with VJM emerging as the most effective among them. The results demonstrate that VJM significantly enhances generalization performance compared to standard GP on \VJMSTDB{} datasets, while it underperforms on only \VJMSTDW{} datasets, highlighting VJM's high efficacy in controlling overfitting. Other methods, aside from P-VRM, perform worse than VJM on more than half of the datasets and exceed VJM on only 0-5 datasets, further confirming the superiority of using VJM as a regularization metric. \cref{fig: Test Plot} provides evolutionary plots of test $R^2$ scores on four representative datasets, which clearly illustrate that VJM effectively controls overfitting. In contrast, methods like IODC suffer from overfitting, with test $R^2$ scores declining slightly in later generations, resulting in inferior performance compared to VJM.
        \begin{figure}[!tb]
            \centering
            \includegraphics[width=\linewidth]{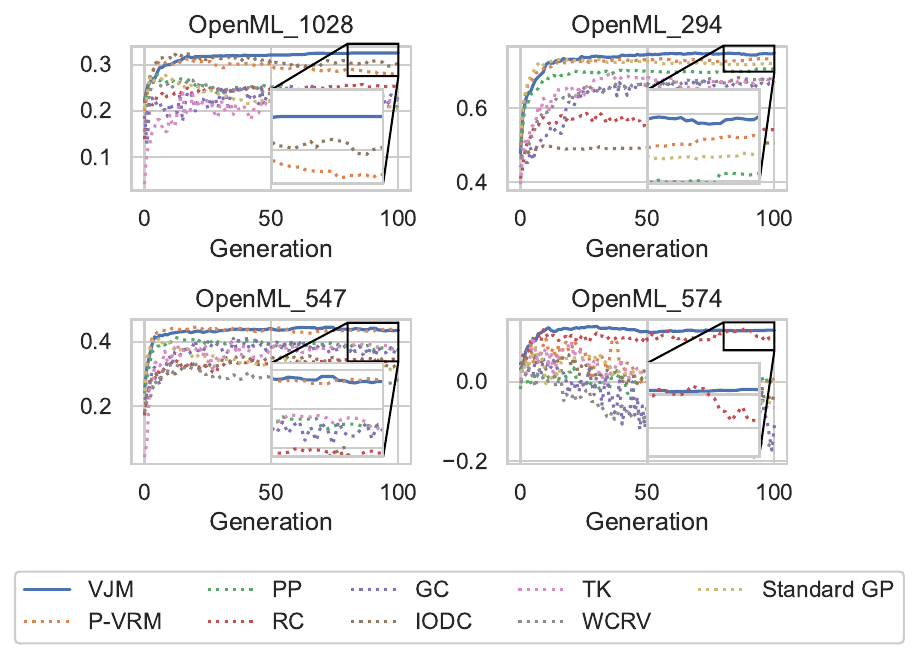}
            \caption{Evolutionary plots of the \textbf{test $R^2$ scores} for different complexity control methods.}
            \label{fig: Test Plot}
        \end{figure}
        \subsubsection{VJM vs VRM}
        \def\VJMPVRMB{22}
        \def\VJMPVRMW{6}
        In the deep learning domain, a common practice is to optimize VRM directly rather than decomposing it into training loss and vicinal Jensen gap and optimizing them simultaneously. However, as evidenced by \cref{tab: Test R2}, VJM outperforms P-VRM on \VJMPVRMB{} datasets and underperforms on only \VJMPVRMW{} dataset. The experimental results suggest that optimizing the vicinal Jensen gap and cross-validation loss on training data is more advantageous than directly optimizing mixup-based VRM. There are four reasons for this advantage. First, decomposing mixup-based VRM into training loss and vicinal Jensen gap allows for a flexible balance between these objectives through a noise estimation mechanism. \cref{sec: Noise Estimation} further illustrates the benefits of this approach. Second, VRM estimation can be imprecise in scenarios with limited vicinal risk sampling. By separating the empirical loss $(\mathbf{Y} - \mathbf{X} \boldsymbol{w})^2$ from VRM, we obtain a more accurate estimate of the empirical loss. Third, P-VRM, which jointly optimizes both training error and VRM, is similar to $\alpha$-dominance-based multi-objective optimization. Since VRM includes both training error and a regularization term, the empirical loss may be overemphasized, potentially obstructing the discovery of low-complexity solutions and negatively affecting generalization performance. Fourth, the loss value
        incorporated in vicinal risk within P-VRM is empirical loss rather than cross-validation loss, which may lead to an overly optimistic estimation of generalization performance and increase the risk of overfitting. Therefore, optimizing the vicinal Jensen gap and cross-validation loss on training data proves to be a more effective strategy than directly optimizing VRM.
        \subsubsection{VJM vs Parsimony Pressure}
        \def\VJMPPB{29}
        \def\VJMPPW{2}
        The results in \cref{tab: Test R2} indicate that controlling tree size remains an effective measure for managing overfitting in evolutionary feature construction. In evolutionary feature construction, small GP trees can have good fitting capabilities as long as they are complementary, which differs from single-tree-based symbolic regression methods~\cite{chen2020rademacher}. However, \cref{tab: Test R2} suggests that while PP is a competitive method for overfitting control, it is not as effective as VJM. VJM outperforms PP on \VJMPPB{} datasets and performs worse on only \VJMPPW{} datasets. These results suggest that relying solely on tree size as a measure of model complexity is insufficient, highlighting the importance of controlling functional complexity for effective overfitting control.
        \subsection{Comparisons on Training $R^2$ Scores}
        \begin{table*}[!tb]
            \centering
            \caption{Statistical comparison of \textbf{training $R^2$ scores} when optimizing various model complexity measures.}
            \label{tab: Training R2}
            \resizebox{\textwidth}{!}{
                \begin{tabular}{ccccccccc}
                    \toprule
                    & \textbf{P-VRM}          & \textbf{PP}              & \textbf{RC}            & \textbf{GC}              & \textbf{IODC}            & \textbf{TK}              & \textbf{WCRV}& \textbf{Standard GP}\\
                    \midrule
                    \textbf{VJM}   & 9(+)/17($\sim$)/32({-}) & 21(+)/7($\sim$)/30({-})  & 49(+)/6($\sim$)/3({-})& 28(+)/15($\sim$)/15({-})& 24(+)/14($\sim$)/20({-})& 30(+)/8($\sim$)/20({-})& 28(+)/16($\sim$)/14({-})& 0(+)/2($\sim$)/56({-})\\
                    \textbf{P-VRM} & ---                     & 19(+)/22($\sim$)/17({-}) & 54(+)/3($\sim$)/1({-}) & 39(+)/12($\sim$)/7({-})& 29(+)/19($\sim$)/10({-})& 29(+)/18($\sim$)/11({-})& 33(+)/16($\sim$)/9({-})& 1(+)/4($\sim$)/53({-})\\
                    \textbf{PP}    & ---                     & ---                      & 57(+)/1($\sim$)/0({-}) & 44(+)/13($\sim$)/1({-})  & 30(+)/26($\sim$)/2({-})& 30(+)/24($\sim$)/4({-})& 36(+)/19($\sim$)/3({-})& 0(+)/2($\sim$)/56({-})\\
                    \textbf{RC}    & ---                     & ---                      & ---                    & 0(+)/3($\sim$)/55({-})   & 3(+)/6($\sim$)/49({-})   & 0(+)/6($\sim$)/52({-})& 0(+)/13($\sim$)/45({-})& 0(+)/0($\sim$)/58({-})\\
                    \textbf{GC}    & ---                     & ---                      & ---                    & ---                      & 12(+)/31($\sim$)/15({-}) & 16(+)/24($\sim$)/18({-}) & 21(+)/26($\sim$)/11({-})& 0(+)/1($\sim$)/57({-})\\
                    \textbf{IODC}  & ---                     & ---                      & ---                    & ---                      & ---                      & 15(+)/23($\sim$)/20({-}) & 20(+)/24($\sim$)/14({-})& 0(+)/0($\sim$)/58({-})\\
                    \textbf{TK}    & ---                     & ---                      & ---                    & ---                      & ---                      & ---                      & 23(+)/21($\sim$)/14({-}) & 0(+)/4($\sim$)/54({-}) \\
                    \textbf{WCRV}  & ---                     & ---                      & ---                    & ---                      & ---                      & ---                      & ---                      & 0(+)/0($\sim$)/58({-}) \\
                    \bottomrule
                \end{tabular}
            }
        \end{table*}
        \def\VJMSTDTrainingWorse{56}
        \def\VJMPPTrainingWorse{30}
        \def\VJMGCTrainingWorse{15}
        The experimental results comparing training $R^2$ scores are shown in \cref{tab: Training R2}. These results reveal that standard GP achieves significantly higher training $R^2$ scores compared to VJM, with superior performance on \VJMSTDTrainingWorse{} datasets. \cref{fig: Training Plot} further confirms the superiority of standard GP in terms of training $R^2$ scores. However, as indicated in \cref{tab: Test R2}, standard GP suffers from severe overfitting despite its higher training $R^2$. This shows that solely optimizing training $R^2$ scores is insufficient for achieving good generalization in evolutionary feature construction methods. \cref{fig: Training Test Plot SAM} and \cref{fig: Training Test Plot LOOCV} illustrate the evolutionary plots of corresponding training and test $R^2$ scores. As shown in \cref{fig: Training Test Plot SAM}, training and test $R^2$ scores are highly correlated when using VJM. In contrast, \cref{fig: Training Test Plot LOOCV} demonstrates that improvements in training $R^2$ scores with standard GP do not consistently lead to better test $R^2$ scores and can even worsen performance, indicating overfitting in feature construction. Although controlling model complexity can improve generalization, it is also worth noting that complexity control should be moderate. Overly strict penalization can negatively impact generalization performance. For instance, methods like PP and GC exhibit lower training $R^2$ scores than VJM on \VJMPPTrainingWorse{} and \VJMGCTrainingWorse{} datasets, respectively. As shown in \cref{tab: Test R2}, these complexity control methods lead to worse generalization compared to VJM. Therefore, effective overfitting control requires a careful definition of model complexity, as overly stringent assumptions may over-penalize useful models and yield suboptimal results.
        \begin{figure}[!tb]
            \centering
            \includegraphics[width=\linewidth]{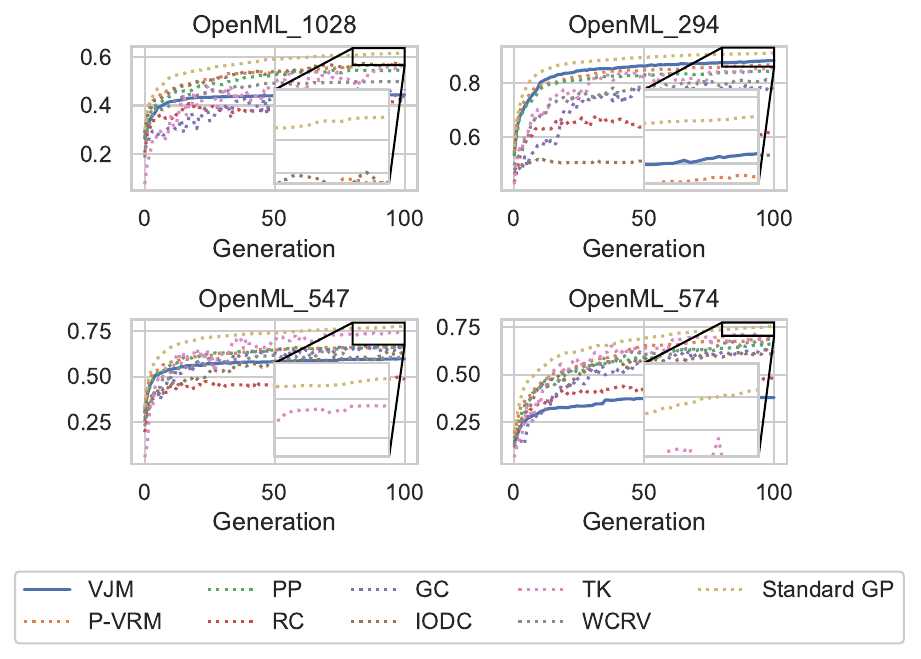}
            \caption{Evolutionary plots of the \textbf{training $R^2$ scores} for various complexity control methods.}
            \label{fig: Training Plot}
        \end{figure}
        \begin{figure}[!tb]
            \centering
            \includegraphics[width=\linewidth]{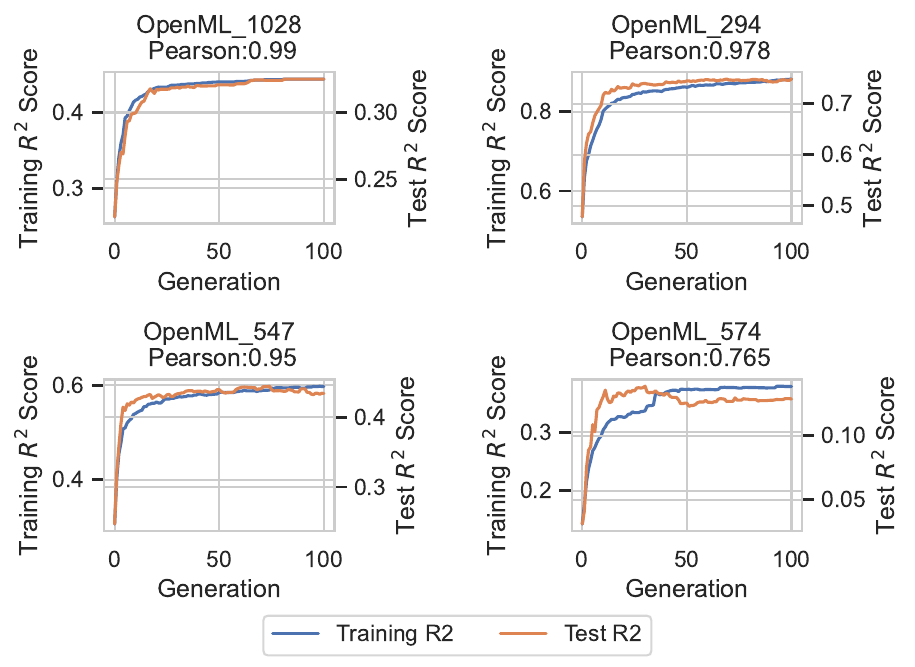}
            \caption{Evolutionary plots of the \textbf{training and test $R^2$ scores} for VJM-GP. Note that the plot includes two axes.}
            \label{fig: Training Test Plot SAM}
        \end{figure}
        \begin{figure}[!tb]
            \centering
            \includegraphics[width=\linewidth]{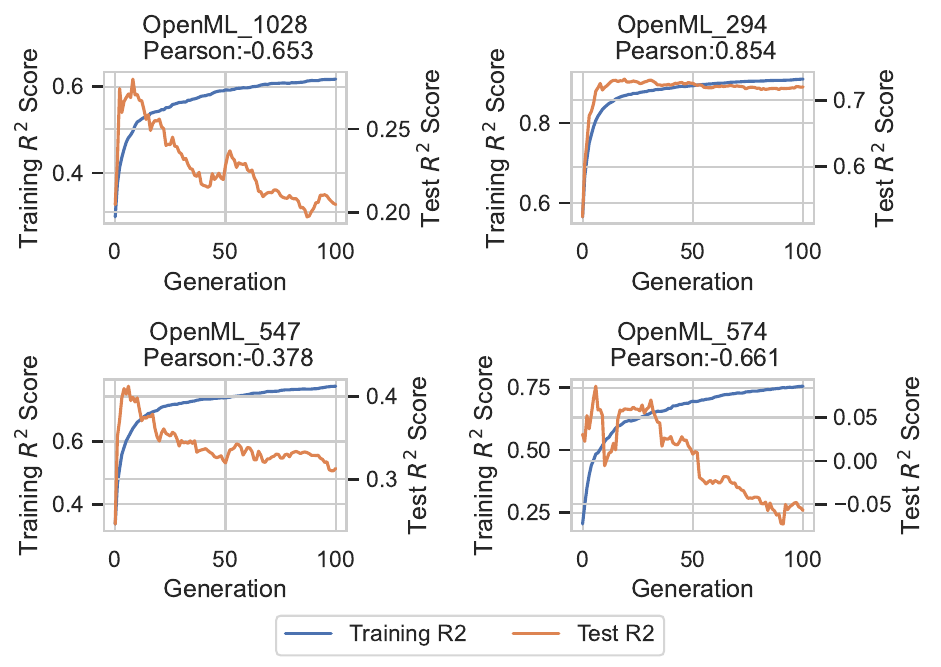}
            \caption{Evolutionary plots of the \textbf{training and test $R^2$ scores} for standard GP.}
            \label{fig: Training Test Plot LOOCV}
        \end{figure}
        \subsection{Comparisons on Tree Size}
        \begin{figure}[!tb]
            \centering
            \begin{subfigure}[t]{0.48\linewidth}
                \centering
                \includegraphics[width=\linewidth, trim=6pt 6pt 6pt 6pt, clip]{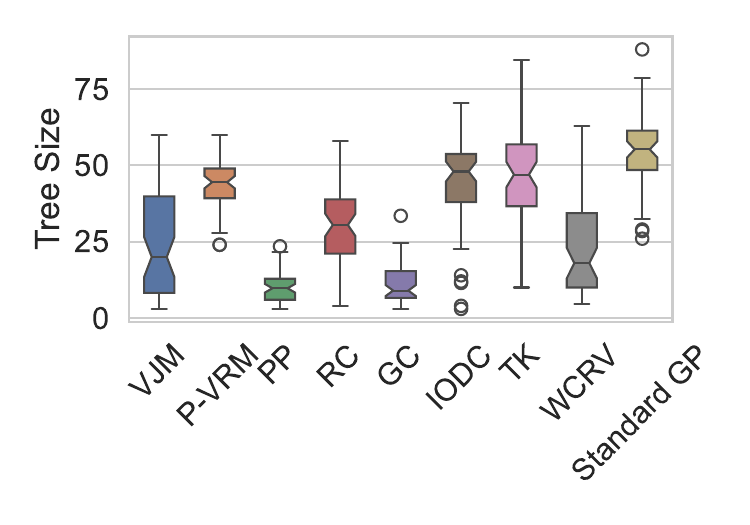}
                \vspace{-1cm}
                \caption{Tree Sizes}
                \label{fig: CT-Complexity}
            \end{subfigure}
            \hfill
            \begin{subfigure}[t]{0.48\linewidth}
                \centering
                \includegraphics[width=\linewidth, trim=6pt 6pt 6pt 6pt, clip]{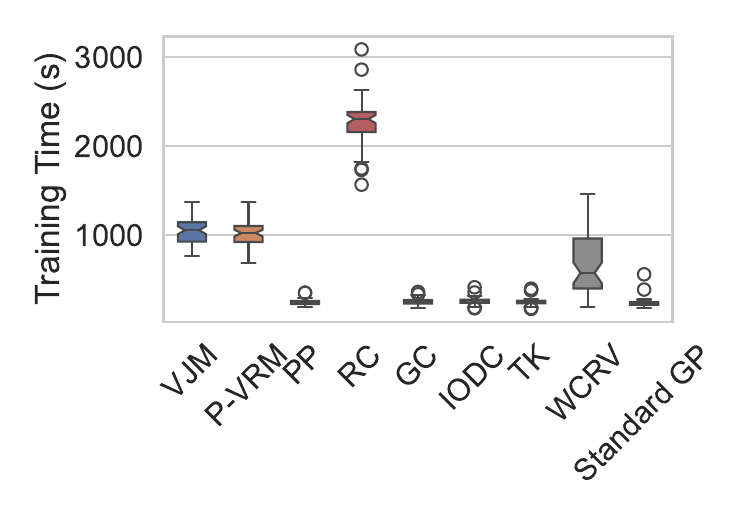}
                \vspace{-1cm}
                \caption{Training Time}
                \label{fig: CT-Training Time}
            \end{subfigure}
            \caption{Comparison of \textbf{tree sizes} and \textbf{training times} over 58 datasets when optimizing different model complexity measures.}
            \label{fig: CT}
        \end{figure}
        \def\STDSize{55.25}
        \def\VJMSize{20}
        A comparison of tree sizes is shown in \cref{fig: CT-Complexity}. The experimental results demonstrate that different complexity measures lead to models of varying sizes. VJM can reduce the model size compared to standard GP. Specifically, the median model size decreases from \STDSize{} to \VJMSize{}. This suggests that VJM implicitly favors simpler models, which could enhance interpretability. An interesting observation is that optimizing WCRV can result in similar tree sizes; however, the results in \cref{tab: Test R2} indicate that optimizing WCRV achieves much worse test performance compared to VJM. This suggests that improving generalization is more closely related to reducing functional complexity rather than merely minimizing model size.
        \subsection{Comparisons on Training Time}
        \label{sec: Training Time}
        \def\VJMTime{1041}
        \def\PPTime{250}
        Comparisons of training time for different complexity measures are shown in \cref{fig: CT-Training Time}. The results indicate that optimizing semantic complexity measures, such as VJM, requires more time than optimizing syntactic complexity measures, like PP. On average, VJM takes \VJMTime{} seconds, whereas PP takes only \PPTime{} seconds. This difference arises from the need to construct features based on vicinal data in VJM, which increases the computational cost.
        Theoretically, the time complexity of standard GP is $O(|P|G)$, where $|P|$ is the population size and $G$ is the number of generations. For VJM-GP, the time complexity increases to $O(|P|G (K+1))$, as each individual is evaluated on the vicinal data for $K$ rounds of vicinal Jensen gap estimation, significantly increasing the training time. However, in practice, not all individuals require $K$ rounds of estimation. For less promising individuals, fewer estimation rounds may suffice. A discussion on an early stopping strategy to accelerate VJM-GP is provided in \Cref{sec: early stop} of the supplementary material.
        Moreover, since overfitting cannot be effectively addressed by simply increasing the number of generations, the additional training time required for VJM is considered a worthwhile trade-off.
        \subsection{Comparisons with Other Learning Algorithms}
        In this section, we follow the evaluation protocol of state-of-the-art symbolic regression benchmarks~\cite{cava2021contemporary} to compare our proposed method against popular machine learning and symbolic regression methods. The training and test data are split according to the same protocol used throughout this paper. To ensure that baseline algorithms are well-tuned, their hyperparameters are optimized on the training data using the parameter grid provided in \cref{fig: SRBench ML-Hyperparameters} in the supplementary material. Given the limited number of samples, hyperparameters are tuned using grid search~\cite{haider2023shape} rather than halving grid search~\cite{de2023transformation}.
        The results presented in \cref{fig: SRBench} demonstrate that our proposed method achieves competitive test $R^2$ scores compared to 15 existing symbolic regression and machine learning methods. The existing popular GP-based regression method, Operon~\cite{burlacu2020operon}, performs poorly when the number of samples is limited. In comparison, our proposed method effectively controls overfitting and achieves the best performance. Nevertheless, the proposed objective function is general and could be integrated into existing GP-based regression methods to enhance their performance. Compared to popular gradient boosting decision trees such as XGBoost~\cite{chen2016xgboost}, VJM-GP also demonstrates very good performance. A Wilcoxon signed-rank test with the Benjamini \& Hochberg correction indicates that VJM-GP significantly outperforms XGBoost, as shown in \cref{fig: Pairwise Comparison}. Furthermore, when compared to deep learning methods designed for small tabular datasets, such as ExcelFormer~\cite{chen2024can}, which leverages the mixup technique to boost generalization, VJM-GP achieves notably better results. This highlights the strength of evolutionary feature construction in scenarios with limited data. Regarding model size, VJM-GP is an order of magnitude smaller than XGBoost, highlighting its potentially superior interpretability. In terms of training time, the proposed evolutionary feature construction method is time-consuming, but average training time remains within two hours. In practical applications, the algorithm could be sped up with early stopping if the best fitness does not improve over a consecutive number of iterations. Overall, these results demonstrate that VJM not only significantly improves the generalization performance of evolutionary feature construction methods but also enables GP to achieve top performance compared to popular machine learning and symbolic regression methods in scenarios where the number of samples is limited.
        \begin{figure*}[!tb]
            \centering
            \includegraphics[width=0.75\linewidth]{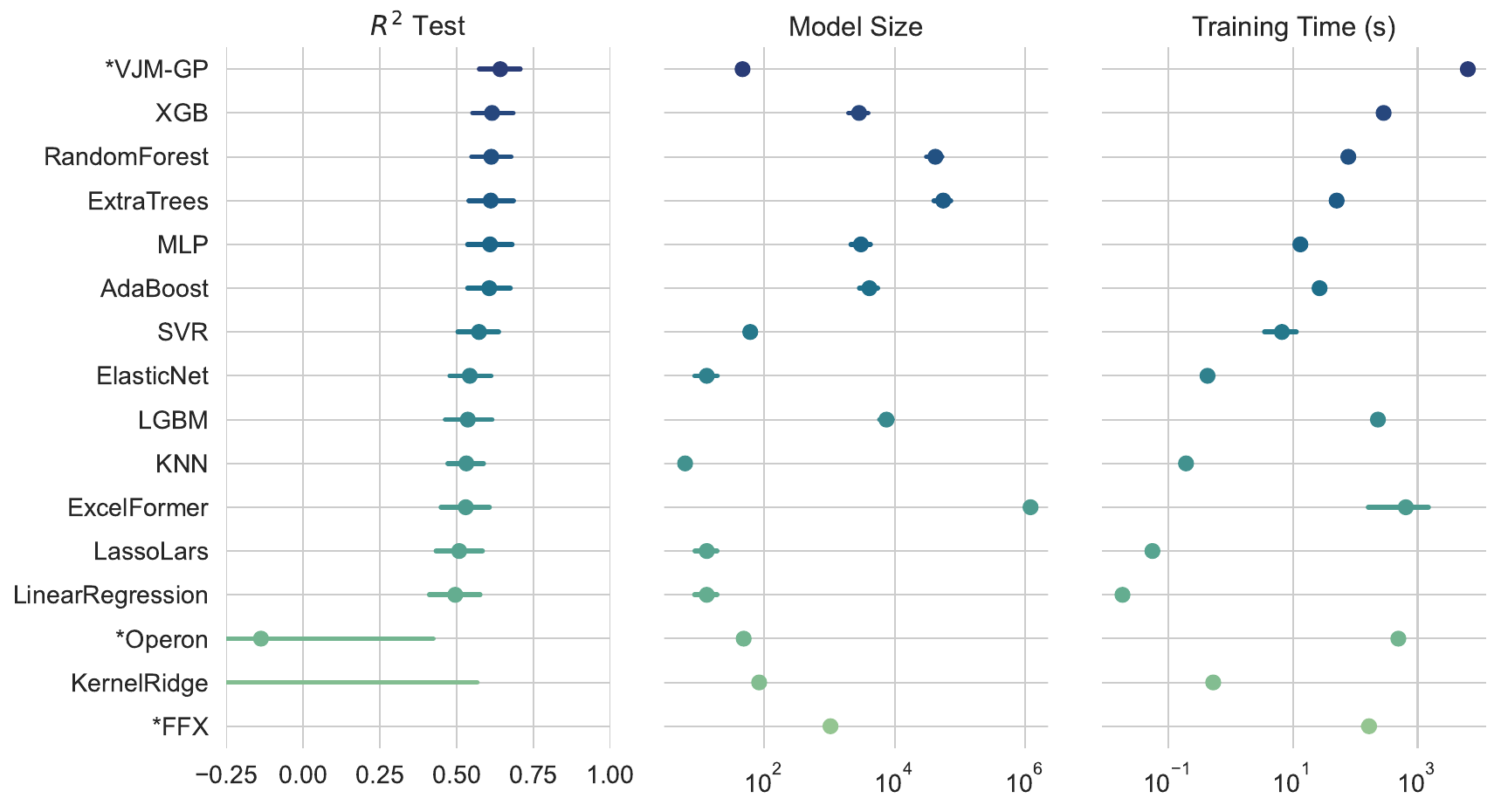}
            \caption{\textbf{Mean $R^2$ scores, model sizes, and training time} of various algorithms across 58 regression problems. (Stars indicate symbolic regression algorithms.)}
            \label{fig: SRBench}
        \end{figure*}
        \begin{figure}[!tb]
            \centering
            \includegraphics[width=\linewidth]{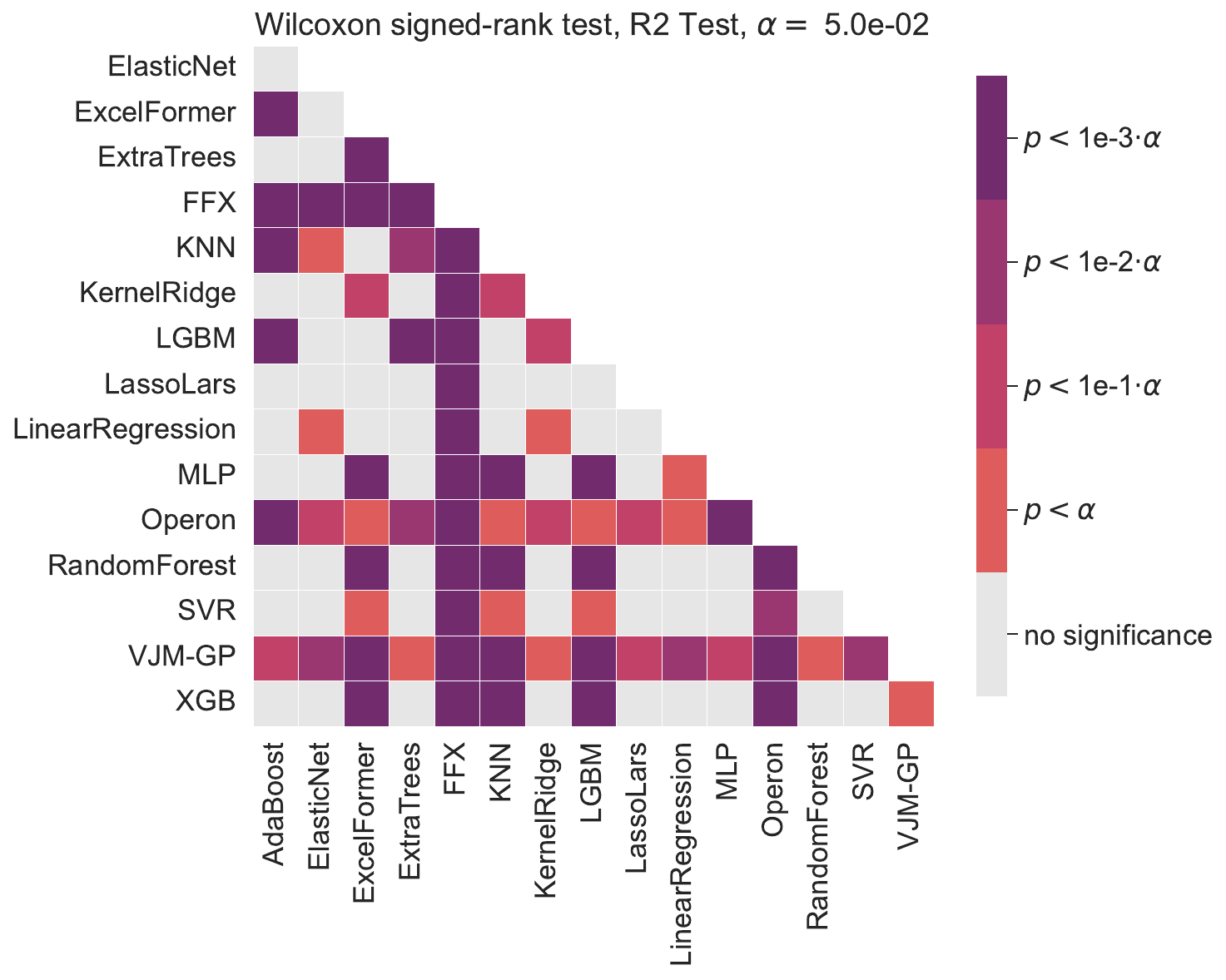}
            \caption{Pairwise statistical comparisons of \textbf{test $R^2$ scores} on 58 regression problems.}
            \label{fig: Pairwise Comparison}
        \end{figure}
        \section{Further Analysis}
        \label{sec: Further Analysis}
        In this section, we first examine the effectiveness of VJM for learning from noisy labels, another important scenario where overfitting frequently occurs. Then, we delve into two key components of the proposed VJM-GP algorithm: the manifold intrusion detection strategy and the noise estimation strategy in the model selection stage, to study the effect of these two improvements made to VJM-GP. Additionally, we provide a visualization of the final constructed features to offer an intuitive comparison between GP trees constructed by VJM-GP and those constructed by standard GP.
        \subsection{Label Noise Learning}
        \def\VJMGPNoiseBetter{31}
        \def\VJMPPNoiseBetter{24}
        \def\VJMPVRMNoiseBetter{19}
        \def\VJMGPNoiseSimilar{3}
        \def\VJMPPNoiseSimilar{9}
        \def\VJMPVRMNoiseSimilar{14}
        \def\VJMGPNoiseWorse{2}
        \def\VJMPPNoiseWorse{3}
        \def\VJMPVRMNoiseWorse{3}
        Learning with limited samples and labels corrupted by noise are two challenging scenarios where evolutionary feature construction tends to overfit. In this section, we demonstrate the effectiveness of the proposed vicinal Jensen gap minimization in controlling overfitting when learning with noisy labels. We focus on 36 datasets with fewer than 2000 instances due to computational constraints. For these datasets, 80\% of samples are randomly chosen as training data, with the remaining 20\% used for testing. Training labels are corrupted with noise drawn from a Gaussian distribution $\mathcal{N}(0,1)$, while test labels remain uncorrupted. In addition to standard GP, we compare VJM with two baseline complexity measures that performed well in \cref{sec: Comparison}: P-VRM and PP. A summary of statistical comparisons is shown in \cref{fig: Noise Learning}. The experimental results indicate that VJM outperforms P-VRM, PP, and standard GP in the context of learning with noisy labels. Specifically, VJM performs significantly better than PP on \VJMPPNoiseBetter{} datasets, similarly on \VJMPPNoiseSimilar{} datasets, and worse on only \VJMPPNoiseWorse{} datasets. VJM also performs significantly better than standard GP on \VJMGPNoiseBetter{} datasets, similarly on \VJMGPNoiseSimilar{} datasets, and worse on \VJMGPNoiseWorse{} datasets. Compared to P-VRM, VJM outperforms it on \VJMPVRMNoiseBetter{} datasets, performs similarly on \VJMPVRMNoiseSimilar{} datasets, and is worse on \VJMPVRMNoiseWorse{} datasets. These results confirm that VJM effectively controls overfitting and enhances generalization performance when learning with noisy labels.
        \begin{figure}[!tb]
            \centering
            \includegraphics[width=\linewidth]{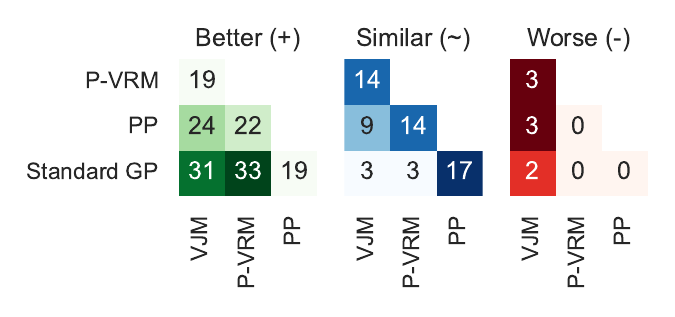}
            \caption{Statistical comparison of \textbf{test $R^2$ scores} when learning with noisy labels for different algorithms.}
            \label{fig: Noise Learning}
        \end{figure}
        \subsection{Manifold Intrusion Detection}
        \def\IntrusionBetter{6}
        \def\IntrusionWorse{2}
        \def\IntrusionTrainingBetter{16}
        \def\IntrusionTrainingWorse{0}
        This section presents ablation studies on the manifold intrusion detection strategy. The experimental results of the statistical comparison on the training set and test set are shown in \subfigref{fig: Intrusion-Training R2} and \subfigref{fig: Intrusion-Test R2}, respectively. These results indicate that the proposed manifold intrusion detection strategy significantly improves the test $R^2$ scores on \IntrusionBetter{} datasets and performs worse on only \IntrusionWorse{} datasets. When examining the training $R^2$ scores, the manifold intrusion detection mechanism improves performance on \IntrusionTrainingBetter{} datasets and does not worsen performance on any dataset. These results confirm that the improvements from manifold intrusion detection are primarily driven by the reduction of regularization effects introduced by incorrectly synthesized samples, which boosts training performance and, in turn, improves generalization.
        \begin{figure}[!tb]
            \begin{subfigure}{0.45\linewidth}
                \centering
                \includegraphics[width=\linewidth]{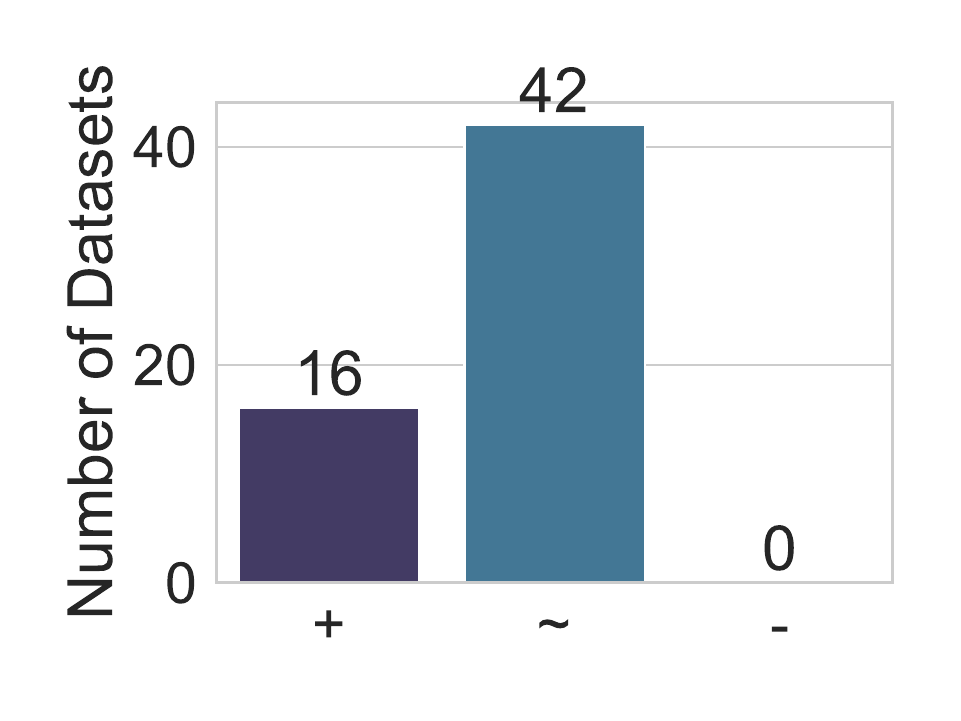}
                \caption{Training $R^2$ scores.}
                \label{fig: Intrusion-Training R2}
            \end{subfigure}
            \begin{subfigure}{0.45\linewidth}
                \centering
                \includegraphics[width=\linewidth]{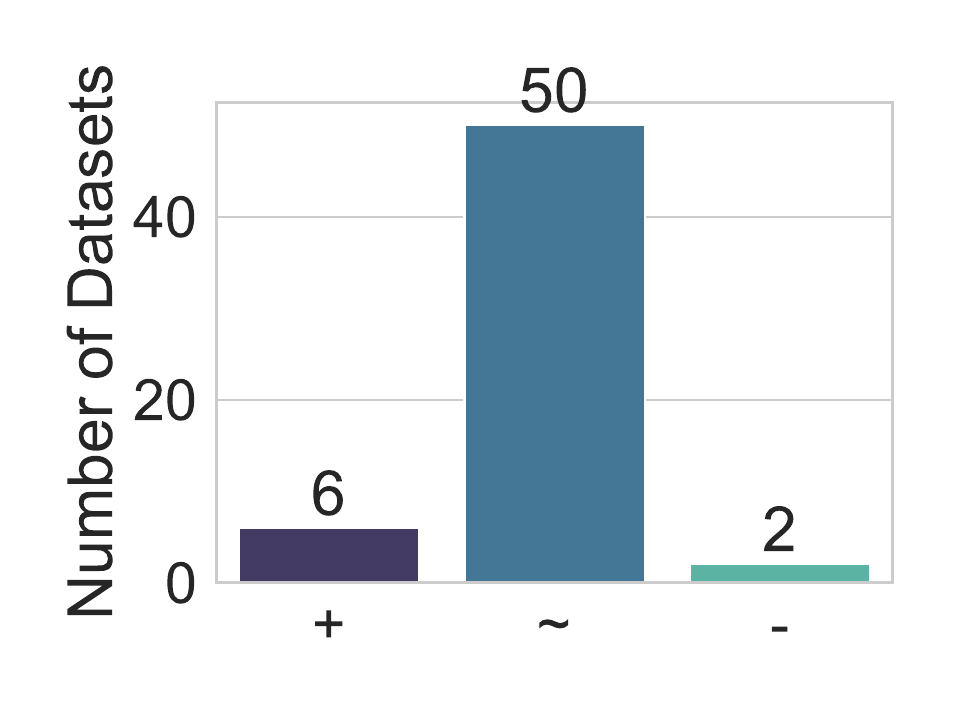}
                \caption{Test $R^2$ scores.}
                \label{fig: Intrusion-Test R2}
            \end{subfigure}
            \caption{Performance comparison with and without manifold intrusion detection. (``+",``$\sim$", and ``-" represent the number of datasets where using manifold intrusion detection performs better, similar to, or worse than not using it.)}
            \label{fig: Intrusion}
        \end{figure}
        \subsection{Noise Estimation Strategy}
        \label{sec: Noise Estimation}
        \def\AdaptiveWeightBetter{12}
        \def\AdaptiveWeightWorse{one}
        In this paper, we propose decomposing vicinal risk into empirical loss and regularization terms. We then design a noise estimation strategy based on the cross-validation loss of extremely randomized trees to determine the coefficient for weighting the vicinal Jensen gap. The experimental results for training $R^2$ scores and test $R^2$ scores with and without the proposed adaptive strategy are shown in \subfigref{fig: Adaption-Training R2} and \subfigref{fig: Adaption-Test R2}, respectively. These results show that adjusting the regularization weight based on dataset characteristics leads to significantly better performance on \AdaptiveWeightBetter{} datasets and only worse performance on \AdaptiveWeightWorse{} dataset. These results validate the effectiveness of our noise estimation strategy and demonstrate the necessity of decomposing vicinal risk into empirical loss and the vicinal Jensen gap. Without such decomposition, there would be no opportunity to set a weight.
        \begin{figure}[!tb]
            \begin{subfigure}{0.45\linewidth}
                \centering
                \includegraphics[width=\linewidth]{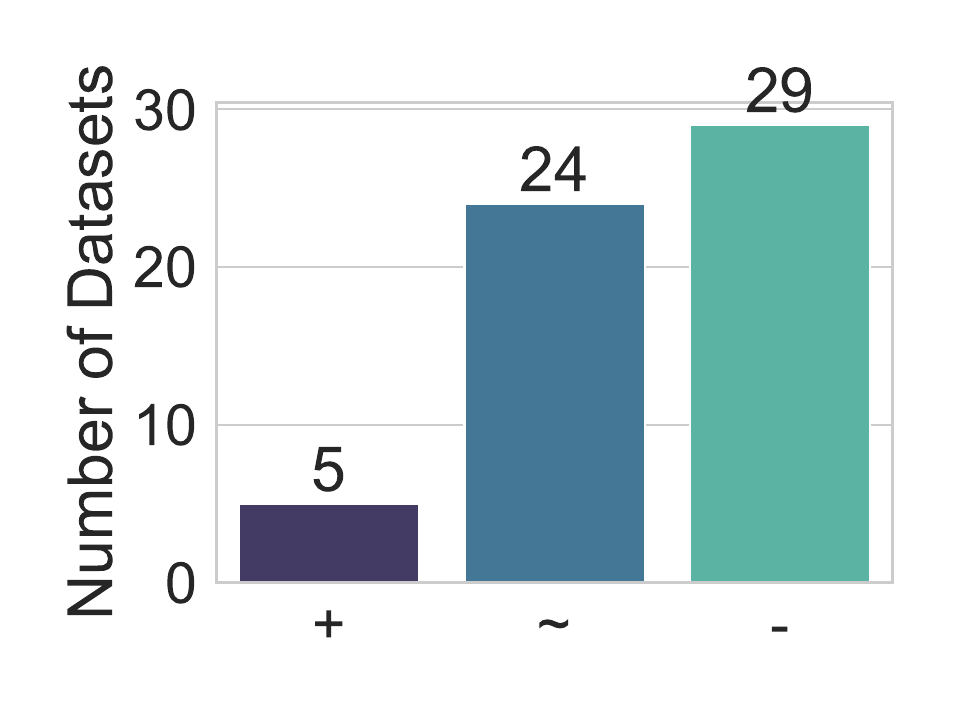}
                \caption{Training $R^2$ scores.}
                \label{fig: Adaption-Training R2}
            \end{subfigure}
            \begin{subfigure}{0.45\linewidth}
                \centering
                \includegraphics[width=\linewidth]{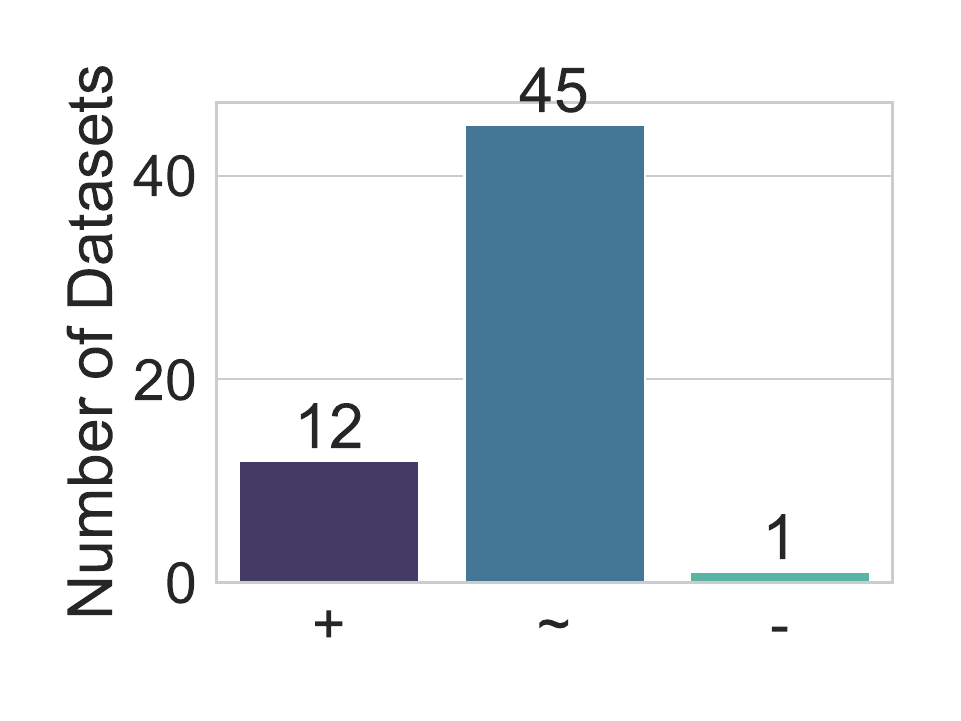}
                \caption{Test $R^2$ scores.}
                \label{fig: Adaption-Test R2}
            \end{subfigure}
            \caption{Performance comparison with and without the noise estimation strategy. (``+",``$\sim$", and ``-" represent the number of datasets where using the noise estimation strategy performs better, similar, or worse, respectively, compared to not using it.)}
            \label{fig: Adaption}
        \end{figure}
        \subsection{Visualization of the Final Models}
        \label{sec: Visualization}
        \def\TRVJM{0.74}
        \def\TRSTD{0.82}
        \def\TEVJM{0.46}
        \def\TESTD{0.36}
        \Cref{tab: VJM Model} and \Cref{tab: GP Model} present a visualization of the final constructed features and their corresponding coefficients in the linear model for the dataset ``OpenML\_547'', representing the constructed features using VJM-GP and standard GP, respectively. This real-world dataset has practical value. Specifically, task aims to predict the concentration of NO2 at each hour in Oslo, Norway. The first original feature, X0, represents the number of cars per hour, which is intuitively positively correlated with the concentration of NO2.
        \par
        The features constructed by standard GP include subtrees such as \texttt{Cos(Min(0.93, X0))}, which lack plausible justification, as there is no reasonable basis for the number of cars per hour to exhibit a sinusoidal correlation with NO2 concentration. In contrast, the features constructed by VJM-GP involve expressions such as \texttt{Add(X0, 0.01)}, which directly capture the expected positive relationship between NO2 concentration and the number of cars, aligning with domain knowledge.
        \par
        In terms of training performance, as shown in \Cref{tab: VJM Model}, VJM-GP achieves a training $R^2$ score of \TRVJM{}, whereas standard GP achieves \TRSTD{}. However, for test performance, VJM-GP achieves a test $R^2$ score of \TEVJM{}, while standard GP achieves \TESTD{}. These results indicate that although standard GP performs better on the training set, it constructs unrealistic features and overfits the training data, leading to poor generalization to unseen data. These findings confirm that VJM-GP can discover features that align with domain knowledge, resulting in superior generalization performance.
        \begin{table}[!tb]
            \centering
            \caption{An example of constructed features and their corresponding coefficients using VJM-GP.}
            \label{tab: VJM Model}
            \resizebox{\columnwidth}{!}{
                \begin{tabular}{lc}
                    \toprule
                    Coefficient & Feature                                                               \\
                    \midrule
                    0.1692      & Max(Neg(X1), Sub(Max(Add(X3, X6), Sub(X1, AQ(Max(X1, X5), X2))), X1)) \\
                    0.3367      & Add(X0, 0.01)                                                         \\
                    -0.1066     & Min(Max(0.45, X5), Add(X5, Min(X2, 0.48)))                            \\
                    0.2449      & Sub(X3, Abs(Min(X4, Sub(Max(Add(X3, X6), X4), 0.01))))                \\
                    -0.0955     & Abs(Abs(Min(Neg(Add(X5, X0)), -0.95)))                                \\
                    -0.2833     & AQ(Sub(Min(X2, X3), Neg(X1)), Max(X4, X2))                            \\
                    0.3367      & X0                                                                    \\
                    0.1146      & Max(X3, X4)                                                           \\
                    -0.0836     & X2                                                                    \\
                    0.0219      & Square(Square(X6))                                                    \\
                    \bottomrule
                \end{tabular}
            }
        \end{table}
        \begin{table}[!tb]
            \centering
            \caption{An example of constructed features and their corresponding coefficients using Standard GP.}
            \label{tab: GP Model}
            \resizebox{\columnwidth}{!}{
                \begin{tabular}{lc}
                    \toprule
                    Coefficient & Feature                                                                          \\
                    \midrule
                    0.5959      & Max(Neg(X1), Sub(Neg(Add(X1, X0)), Sub(Min(X2, Log(Sub(X6, X4))), Add(X3, X6)))) \\
                    0.3167      & Neg(X2)                                                                          \\
                    0.1840      & Log(AQ(X2, X0))                                                                  \\
                    0.1825      & Sin(Abs(X4))                                                                     \\
                    0.1284      & Add(Mul(X1, X4), Neg(X6))                                                        \\
                    0.1437      & Cos(Sub(X6, X4))                                                                 \\
                    -0.1535     & Square(Cos(Min(0.93, X0)))                                                       \\
                    0.9177      & Neg(AQ(Neg(AQ(X0, Sin(Abs(X4)))), Sub(X2, X2)))                                  \\
                    -0.1040     & X5                                                                               \\
                    -0.0605     & Sin(X6)                                                                          \\
                    \bottomrule
                \end{tabular}
            }
        \end{table}
        \section{Conclusions}
        \label{sec: Conclusion}
        In this paper, we decompose vicinal risk into two components: empirical loss and regularization terms, specifically the vicinal Jensen gap and finite difference. This decomposition not only provides deeper insight into the mechanisms behind vicinal risk estimation but also offers an opportunity to design a regularized GP framework for controlling overfitting in evolutionary feature construction algorithms. Additionally, the decomposition enables us to develop a noise estimation strategy to adjust the regularization pressure based on the noise level of datasets. To avoid manifold intrusion from incorrect data augmentation, we propose an intrusion detection strategy. With these improvements, the experimental results show that optimizing training error and the Jensen gap leads to better generalization performance compared with standard GP and training with seven other complexity measures, including directly applying mixup-based VRM to GP without decomposition. Furthermore, our findings suggest that effective overfitting control requires a focus on reducing functional complexity, and merely controlling model size is insufficient. Ablation studies demonstrate that the noise estimation strategy significantly improves generalization performance for datasets with different levels of complexity, and the studies on manifold intrusion detection show the necessity of avoiding manifold intrusion in data synthesis.
        In this work, we primarily explore the effectiveness of the proposed method on GP-based feature construction algorithms. To improve efficiency, the proposed method can be accelerated using caching techniques to avoid re-evaluating performance on repeated GP trees, thereby reducing computational costs. Moreover, with the recent surge in using deep learning to solve symbolic regression tasks, it would be interesting to investigate whether the proposed method can enhance the generalization performance of equation learner-based~\cite{dong2024evolving}, Transformer-based~\cite{kamienny2022end}, or reinforcement learning-based~\cite{mundhenk2021symbolic} symbolic regression methods.
\bibliographystyle{IEEEtranN}
\bibliography{mybibliography}
    \appendices
    \clearpage
        \begin{figure*}
            \begin{center}
                \Huge
                Supplementary Materials for ``Enhancing Generalization in Evolutionary Feature Construction for Symbolic Regression through Vicinal Jensen Gap Minimization''
                \\[10pt]
                \normalsize
                Hengzhe Zhang,
                Qi Chen,~\IEEEmembership{Member,~IEEE,}
                Bing Xue,~\IEEEmembership{Fellow,~IEEE,}
                Wolfgang Banzhaf,~\IEEEmembership{Member,~IEEE,}
                Mengjie Zhang,~\IEEEmembership{Fellow,~IEEE}
            \end{center}
        \end{figure*}
        \section{Proof of Theorem }
        \subsection{Proof of \Cref{the: Perturbation VRM}}
        \begin{proof}
            The squared loss for the vicinal sample, $(y_{\text{vic}} - f(x_{\text{vic}}))^2$, is expanded as follows:
            \begin{equation}
                \begin{split}
                    &(y_{\text{vic}} - f(x_{\text{vic}}))^2 \\
                    &= [y_i - f(x_i + \epsilon)]^2 \\
                    &= \big[y_i - f(x_i) + f(x_i) - f(x_i + \epsilon)\big]^2\\
                    &= (y_i - f(x_i))^2 + 2(y_i - f(x_i))(f(x_i) - f(x_i + \epsilon)) \\
                    &\quad + (f(x_i) - f(x_i + \epsilon))^2\\
                    &\leq  (y_i - f(x_i))^2 + (y_i - f(x_i))^2 \\
                    &\quad +(f(x_i) - f(x_i + \epsilon))^2 + (f(x_i) - f(x_i + \epsilon))^2\\
                    &= 2(y_i - f(x_i))^2 + 2[f(x_i) - f(x_i + \epsilon)]^2.
                \end{split}
            \end{equation}
        \end{proof}
        \subsection{Proof of \Cref{the: VJM}}
        \begin{proof}
            The squared error on the vicinal example is given by $(y_{\text{vic}} - f(x_{\text{vic}}))^2$. We can decompose this error as follows:
            \begin{equation}
                \begin{aligned}
                    (y_{\text{vic}} - f(x_{\text{vic}}))^2 &= \left[ \lambda y_i + (1-\lambda) y_j - f(\lambda x_i + (1-\lambda) x_j) \right]^2 \\
                    &= \left[ \lambda (y_i - f(x_i)) + (1-\lambda) (y_j - f(x_j)) \right. \\
                    &\quad \left. + \left( \lambda f(x_i) + (1-\lambda) f(x_j) - f(x_{\text{vic}}) \right) \right]^2 \\
                    &= \left[ \lambda (y_i - f(x_i)) + (1-\lambda) (y_j - f(x_j)) \right. \\
                    &\quad \left. + \lambda f(x_i) + (1-\lambda) f(x_j) - f(x_{\text{vic}}) \right]^2 \\
                    &= \left[ \lambda (y_i - f(x_i)) + (1-\lambda) (y_j - f(x_j)) \right. \\
                    &\quad \left. + \left( \lambda f(x_i) + (1-\lambda) f(x_j) - f(x_{\text{vic}}) \right) \right]^2 \\
                    &\leq 3\left( \lambda (y_i - f(x_i)) \right)^2 + 3\left( (1-\lambda) (y_j - f(x_j)) \right)^2 \\
                    &\quad + 3\left( \lambda f(x_i) + (1-\lambda) f(x_j) - f(x_{\text{vic}}) \right)^2,
                \end{aligned}
            \end{equation}
            where the inequality follows from the fact that $(a + b + c)^2 \leq 3a^2 + 3b^2 + 3c^2$ for any real numbers $a$, $b$, and $c$.
        \end{proof}
        \section{An Intuitive Example of Vicinal Jensen Gap Regularization}
        \begin{figure}[!tb]
            \centering
            \begin{subfigure}[b]{0.48\linewidth}
                \centering
                \includegraphics[width=\linewidth]{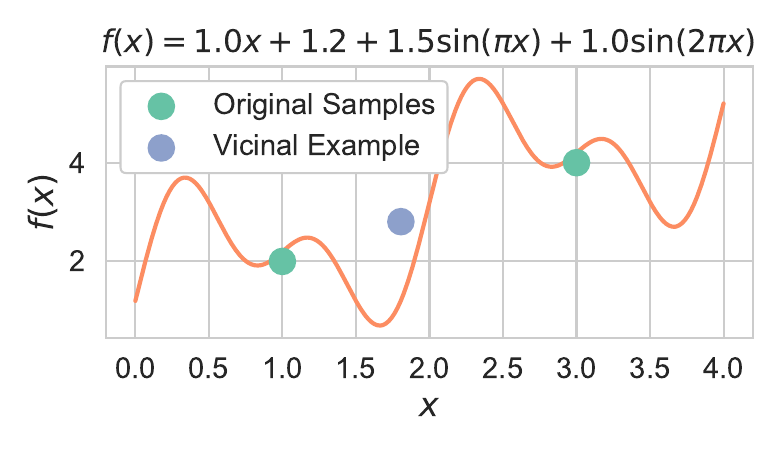}
                \caption{Rugged Symbolic Model}
                \label{fig: samples}
            \end{subfigure}
            \hfill
            \begin{subfigure}[b]{0.48\linewidth}
                \centering
                \includegraphics[width=\linewidth]{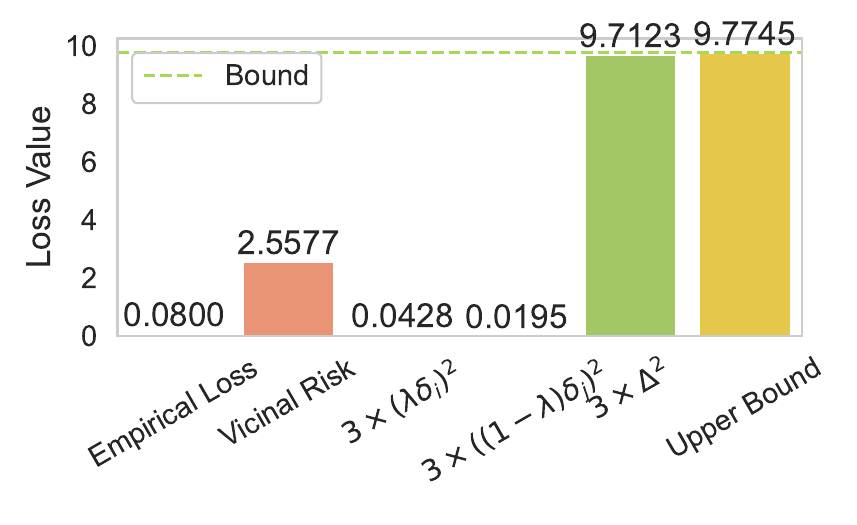}
                \caption{Loss Components}
                \label{fig: loss components}
            \end{subfigure}
            \vfill
            \begin{subfigure}[b]{0.48\linewidth}
                \centering
                \includegraphics[width=\linewidth]{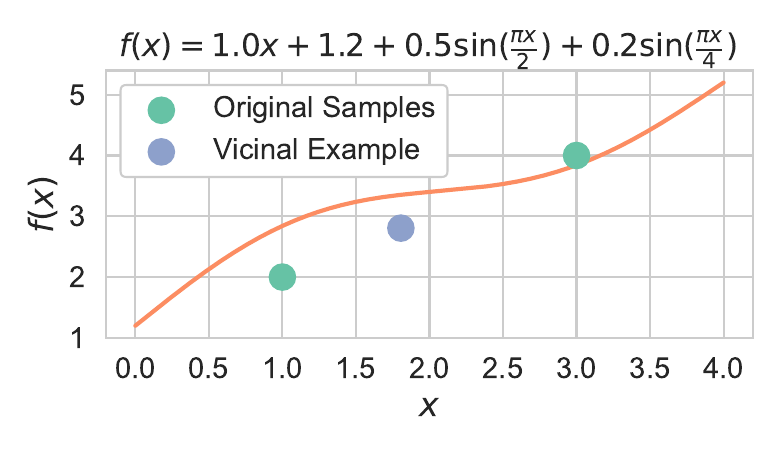}
                \caption{Smooth Symbolic Model}
                \label{fig: smooth samples}
            \end{subfigure}
            \hfill
            \begin{subfigure}[b]{0.48\linewidth}
                \centering
                \includegraphics[width=\linewidth]{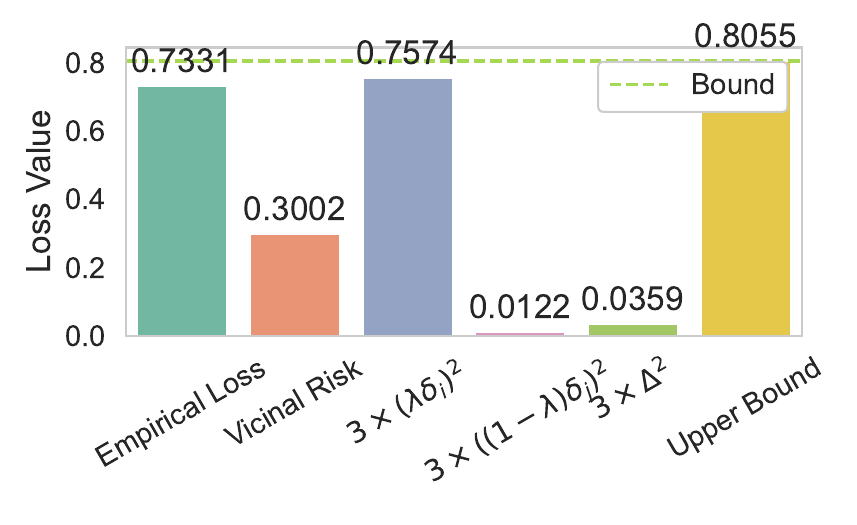}
                \caption{Loss Components}
                \label{fig: smooth loss components}
            \end{subfigure}
            \caption{Comparison of rugged and smooth symbolic models with vicinal Jensen gap decomposition.}
            \vspace{-0.35cm}
        \end{figure}
        To understand how the derived vicinal Jensen gap influences overfitting control, we analyze the vicinal risk decomposition for both a highly complex function and a smooth function, as shown in \cref{fig: loss components} and \cref{fig: smooth loss components}, respectively. For simplicity, we define $\delta_i = y_i - f_{x_i}$, $\delta_j = y_j - f_{x_j}$, and $\Delta = \lambda \cdot f_{x_i} + (1 - \lambda) \cdot f_{x_j} - f_{\text{vic}}$. The training samples, presented in \cref{fig: samples} and \cref{fig: smooth samples}, are consistent across both cases. While the empirical loss, as shown in \cref{fig: loss components} and \cref{fig: smooth loss components}, does not demonstrate the advantage of the smooth function, the vicinal Jensen gap provides clearer insight. The smooth function exhibits a significantly smaller vicinal Jensen gap compared to the complex function, suggesting that minimizing the vicinal Jensen gap can effectively promote smoother functions, thereby reducing overfitting.
        \section{Kernel Analysis}
        \label{sec: Kernel Analysis}
        \def\YSpaceBetter{8}
        \def\YSpaceWorse{5}
        In this paper, the closeness between instances is measured using the kernel defined in \Cref{eq: Kernel}, which simplifies the distance metric $D((x_i, y_i), (x_j, y_j))$ to $D(y_i, y_j)$. For regression problems where the feature space $X$ has low dimensionality, the distance can alternatively be measured using $D((x_i, y_i), (x_j, y_j))$. Specifically, the kernel is defined as:
        \begin{equation}
            D((x_i, y_i), (x_j, y_j)) = \exp\left(-\gamma \lVert (x_i, y_i) - (x_j, y_j) \rVert^2\right),
            \label{eq: XY Kernel}
        \end{equation}
        where $\lVert (x_i, y_i) - (x_j, y_j) \rVert^2$ indicates that the features $X$ are concatenated with the labels $Y$ before computing the distance. This method is referred to as $\text{Mix}_{XY}$, in contrast to $\text{Mix}_{Y}$, which uses only the label $Y$ for distance computation. In theory, distinct assumptions underlie $\text{Mix}_{Y}$ and $\text{Mix}_{XY}$. In $\text{Mix}_{Y}$, two points are sampled based on their $Y$ values, and linear regularization is applied between these points. This allows the algorithm to synthesize instances that may be distant in feature space, with the vicinal Jensen gap applied over a larger region. This is particularly important in cases where instances are often sparsely distributed. For instance, as shown in \cref{fig: XYSpaceLimitation}, consider a one-dimensional case with four samples: $(1, 1)$, $(1.1, 1.1)$, $(5, 1.1)$, and $(5.1, 1.2)$. Here, $\text{Mix}_{Y}$ can synthesize a sample between $(1.1, 1.1)$ and $(5, 1.1)$, achieving regularization, whereas $\text{Mix}_{XY}$ cannot.
        \par
        In contrast, $\text{Mix}_{XY}$ requires closeness in both the $Y$ and $X$ spaces, making manifold intrusion less likely. For example, as shown in \cref{fig: YSpaceLmitation}, in a case with samples $(1, 1)$, $(2, 5)$, and $(3, 1)$, $\text{Mix}_{Y}$ may interpolate between $(1, 1)$ and $(3, 1)$, potentially leading to manifold intrusion. In contrast, $\text{Mix}_{XY}$ is less prone to this issue, as $(1, 1)$ and $(3, 1)$ are far apart when the $X$ space is taken into account. To address the problem of manifold intrusion, a detection mechanism is proposed in \Cref{sec: Manifold Intrusion Detection}. However, this mechanism may not detect all instances of manifold intrusion, as some intrusion points may lie within the confidence bound defined in \Cref{eq: Manifold Intrusion}.
        \par
        The experimental results, presented in \cref{fig: Different Space}, demonstrate that both $\text{Mix}_{Y}$ and $\text{Mix}_{XY}$ are suited for specific types of datasets. Specifically, $\text{Mix}_{Y}$ performs significantly better on \YSpaceBetter{} datasets, but performs worse on \YSpaceWorse{} datasets. This confirms that different distance kernels rely on different assumptions, and selecting the most suitable kernel based on an analysis of the data—such as assessing whether the data are sparsely distributed in feature space—can lead to better results when using VJM-GP.
        \begin{figure}[!tb]
            \centering
            \includegraphics[width=0.5\columnwidth]{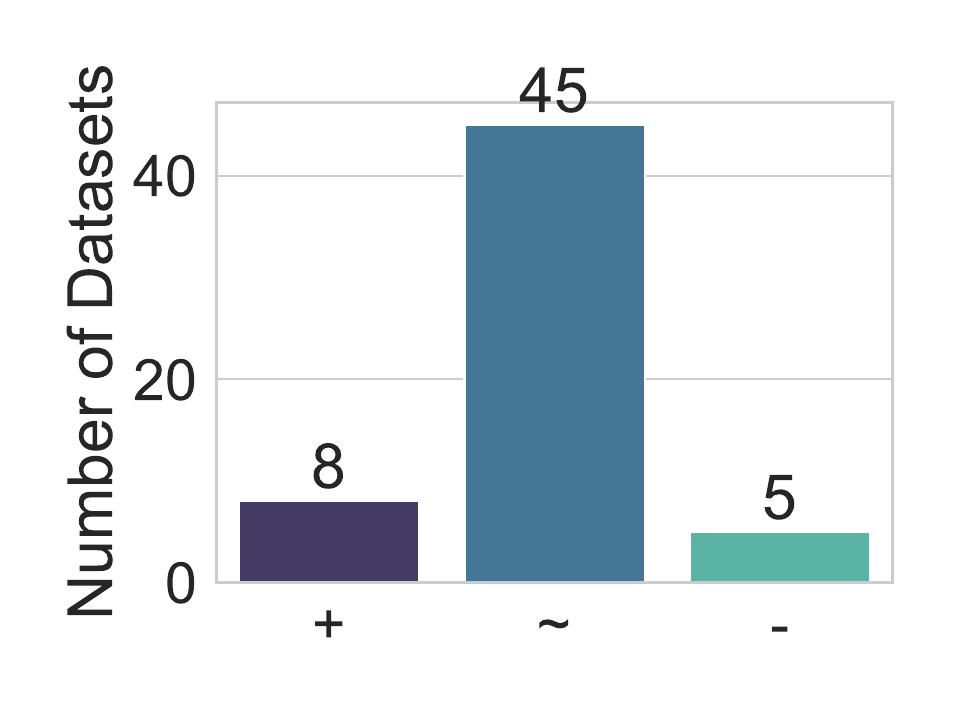}
            \caption{Statistical comparison of test $R^2$ scores when using $\text{Mix}_{Y}$ instead of $\text{Mix}_{XY}$.}
            \label{fig: Different Space}
        \end{figure}
        \begin{figure}[!tb]
            \centering
            \includegraphics[width=0.65\columnwidth]{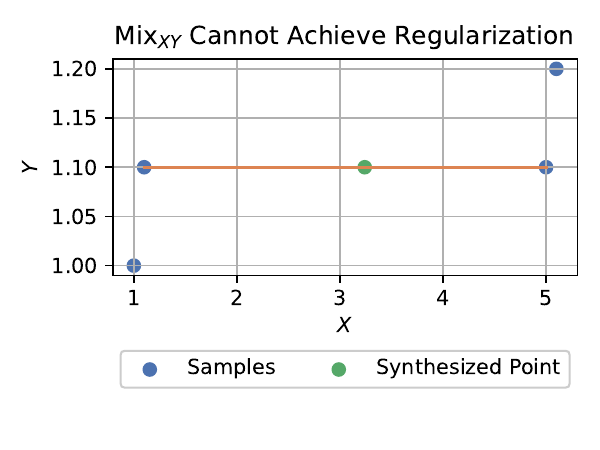}
            \caption{A case where $\text{Mix}_{Y}$ achieves sufficient regularization, but $\text{Mix}_{XY}$ lacks regularization strength. The synthesized point in this example is generated by $\text{Mix}_{Y}$.}
            \label{fig: XYSpaceLimitation}
        \end{figure}
        \begin{figure}[!tb]
            \centering
            \includegraphics[width=0.65\columnwidth]{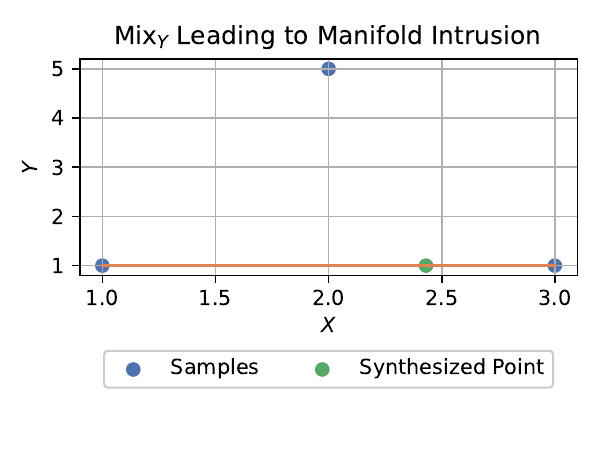}
            \caption{A case where $\text{Mix}_{Y}$ leads to manifold intrusion.}
            \label{fig: YSpaceLmitation}
        \end{figure}
        \section{Comparison with Same Runtime Budget}
        \label{sec: More Generations}
        As shown in \subfigref{fig: CT-Training Time}, the proposed method requires more training time compared to baseline methods such as PP. However, overfitting is a complex issue that cannot be resolved merely by allocating more training time. To investigate this, we compared the proposed method against PP and standard GP under an identical runtime budget. Specifically, all methods were allowed to run for an unlimited number of generations within a runtime limit of 1000 seconds, as shown in \cref{fig: MoreGen}. Under these conditions, the statistical comparison of test $R^2$ scores is presented in \Cref{tab: MoreGen}. The results show that even with an increased number of generations to achieve the same runtime budget, both PP and standard GP remain inferior to VJM-GP. Consequently, the necessity of employing VJM to mitigate overfitting is validated, despite its higher computational cost.
        \begin{figure}[!tb]
            \centering
            \includegraphics[width=0.5\linewidth]{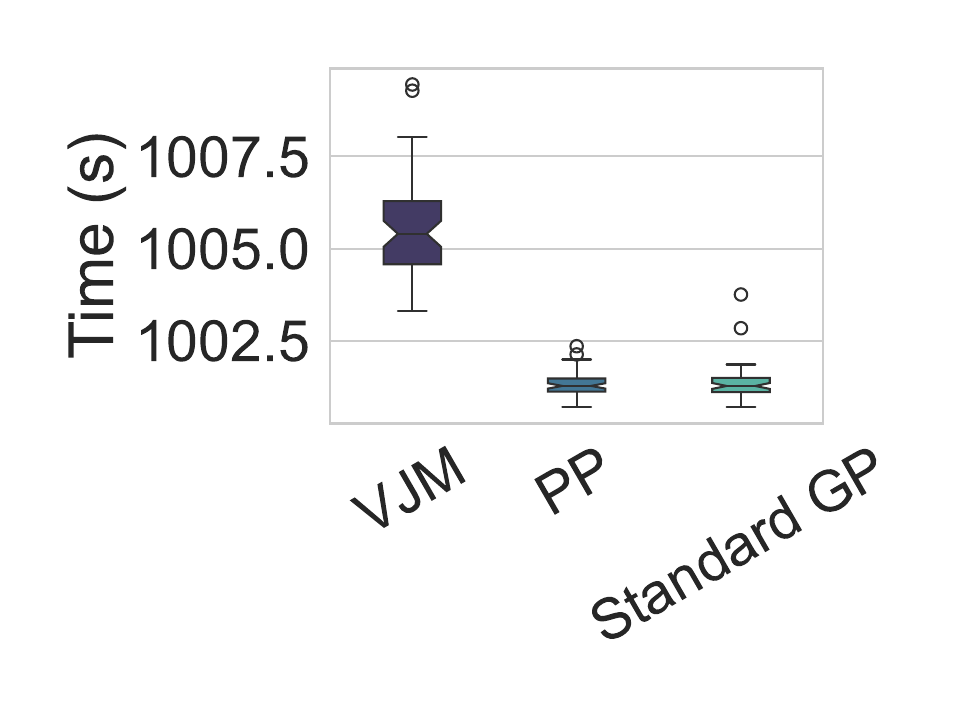}
            \caption{Comparison of training time between the proposed method and baseline methods under the same runtime budget. The minor difference arises because the program waits for the completion of the current generation before terminating after reaching the time limit.}
            \label{fig: MoreGen}
        \end{figure}
        \begin{table}[!tb]
            \centering
            \caption{Statistical comparison of \textbf{test $R^2$ scores} between the proposed method and baseline methods under the same runtime budget.}
            \label{tab: MoreGen}
            \begin{tabular}{ccccc}
                \toprule
                & \textbf{PP}             & \textbf{Standard GP}    \\
                \midrule
                \textbf{VJM} & 34(+)/20($\sim$)/4({-}) & 39(+)/14($\sim$)/5({-}) \\
                \textbf{PP}  & ---                     & 27(+)/25($\sim$)/6({-}) \\
                \bottomrule
            \end{tabular}
        \end{table}
        \section{Speeding Up Vicinal Jensen Gap Estimation with Early Stopping}
        \label{sec: early stop}
        \subsection{Early Stopping Algorithm}
        \label{sec: quick estimation}
        For each synthesized vicinal data point, the feature construction process must be executed, making the vicinal Jensen gap estimation process time-consuming. However, not all individuals require a full evaluation. The following two characteristics of VJM-GP can be leveraged to accelerate this process:
        \begin{itemize}
            \item \textbf{Best Objective-Driven Selection:} The final model is selected based on the individual with the best combination of two objective values, as defined in \Cref{eq: Objective}, over the entire evolutionary process.
            \item \textbf{Monotonicity:} The vicinal Jensen gap, $O_{2}(\Phi)$, monotonically increases with the number of estimation rounds $k$.
        \end{itemize}
        Assume the historically best individual is $\Phi_{Best}$. For $k = 2, \dots, K$, if the current individual $\Phi$ has a combined objective value $O_{1}(\Phi) + \tau \mathcal{V}_{current}$ that is higher than $O_{1}(\Phi_{Best}) + \tau O_{2}(\Phi_{Best})$, the vicinal Jensen gap estimation for $\Phi$ can be terminated early, as the individual cannot outperform the best individual. To ensure that the second objective, $O_{2}(\Phi)$, is not entirely unknown, at least one round of vicinal Jensen gap estimation is conducted for every individual. This is necessary for multi-objective optimization, where all individuals must have defined objective values to guide the evolutionary progress. The pseudocode for this approach is provided in \cref{alg: Vicinal Risk Estimation with Stop}.
        \begin{algorithm}[!tb]
            \caption{Vicinal Jensen Gap/Finite Difference Minimization with Early Stopping}
            \label{alg: Vicinal Risk Estimation with Stop}
            \begin{algorithmic}[1]
                \Require GP Tree $\Phi$, Input $X$, Target Outputs $Y$, Linear Model $F$, Number of Iterations $K$, Current Best Individual $\Phi_{Best}$
                \State Initialize Regularization Term: $\mathcal{V} \gets 0$
                \State $\Phi(X) \gets $ Feature Construction ($X, \Phi$)
                \For{$k = 1, \ldots, K$}
                    \State $X_{vic}, Y_{vic} \gets$ Data Synthesis ($X, Y, k$) \Comment{Cached}
                    \State $\Phi(X_{vic}) \gets $ Feature Construction ($X_{vic}, \Phi$)
                    \State $\hat{Y} \gets$ Prediction($F, \Phi(X_{vic})$)
                    \For{$i = 1, \ldots, N$}
                        \State $\mathcal{V}_i \gets \max(\mathcal{V}_i, (\hat{y}_i^k - y_{vic}^k)^2)$
                    \EndFor
                    \State $\mathcal{V}_{current} = \frac{1}{N} \sum_{i=1}^{N} \mathcal{V}_i$
                    \If{$O_{1}(\Phi) + \tau \mathcal{V}_{current} > O_{1}(\Phi_{Best}) + \tau O_{2}(\Phi_{Best})$}
                        \State \Return $\mathcal{V}_{current}$ \Comment{Terminate Early}
                    \EndIf
                \EndFor
                \Ensure Regularization Term $\mathcal{V} = \frac{1}{N} \sum_{i=1}^{N} \mathcal{V}_i$
            \end{algorithmic}
        \end{algorithm}
        \subsection{Results}
        \label{sec: Speedup}
        \def\SpeedupBetter{4}
        \def\SpeedupWorse{11}
        \def\SpeedupBetterThanPP{29}
        \def\SpeedupWorseThanPP{4}
        \def\VJMT{1058}
        \def\EVJMT{381}
        To examine the efficiency improvement of the proposed early stopping strategy, a comparison of time costs is presented in \cref{fig: Speedup}. The standard VJM-GP takes \VJMT{} seconds on average across 58 datasets, while VJM-GP with the early stopping strategy reduces this to \EVJMT{} seconds on average. This demonstrates that the early stopping strategy significantly reduces computational costs, achieving approximately a threefold acceleration. Thanks to this acceleration, the gap in computational time compared to PP has been largely reduced.
        \par
        Regarding test $R^2$ scores, a comparison is presented in \cref{tab: Speedup}. The proposed method performs significantly better than the version without the speedup strategy on \SpeedupBetter{} datasets and worse on \SpeedupWorse{} datasets, indicating a slight trade-off in accuracy. Compared with PP, the proposed method is still significantly better on \SpeedupBetterThanPP{} datasets and significantly worse on \SpeedupWorseThanPP{} datasets only. These results indicate that the early-stop of VJM-GP is practically useful when training time is constrained. The slight reduction in test $R^2$ scores introduced by the speedup strategy arises because early-stopped individuals may have inaccurately estimated vicinal Jensen gaps. This inaccuracy can limit the evolutionary algorithm's ability to identify models with low vicinal Jensen gaps. Nevertheless, the results demonstrate that the speedup strategy effectively balances computational efficiency and predictive performance. In the future, exploring the use of surrogate models to more accurately estimate the vicinal Jensen gap could further improve both the accuracy and efficiency of the estimation process, potentially leading to better overall results.
        \begin{figure}[!tb]
            \centering
            \includegraphics[width=0.5\linewidth]{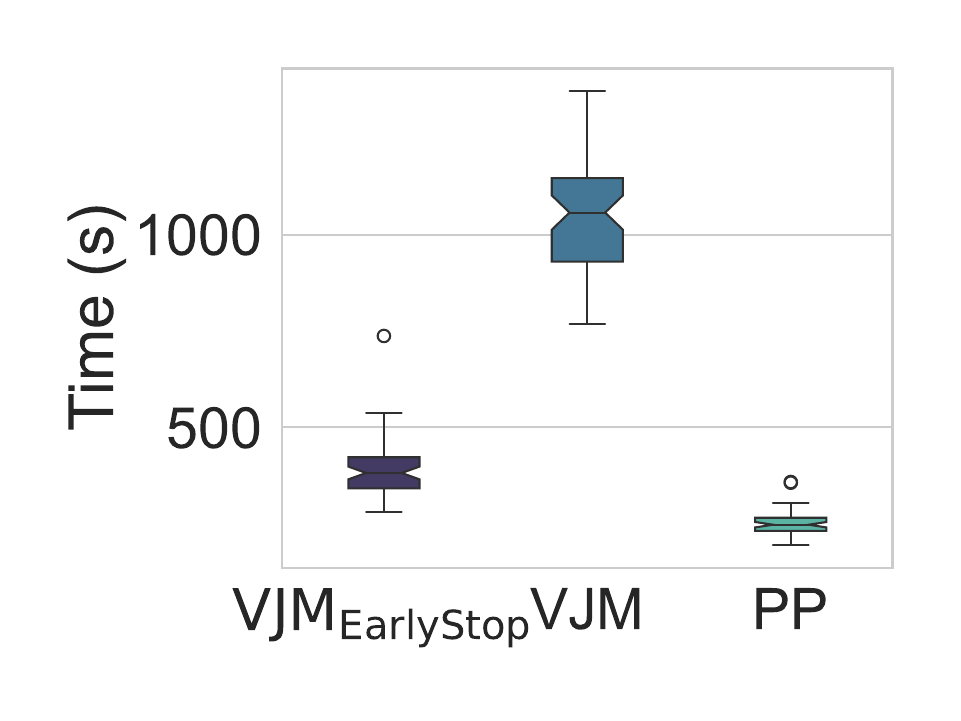}
            \caption{Comparison of training time with and without the early stopping technique.}
            \label{fig: Speedup}
        \end{figure}
        \begin{table}[!tb]
            \centering
            \caption{Statistical comparison of \textbf{test $R^2$ scores} with and without the early stopping technique.}
            \label{tab: Speedup}
            \begin{tabular}{ccccc}
                \toprule
                & \textbf{VJM}            & \textbf{PP}             \\
                \midrule
                \textbf{$\text{VJM}_{\text{EarlyStop}}$} & 4(+)/43($\sim$)/11({-}) & 29(+)/25($\sim$)/4({-}) \\
                \textbf{VJM}                             & ---                     & 29(+)/27($\sim$)/2({-}) \\
                \bottomrule
            \end{tabular}
        \end{table}
        \section{Selection Operators}
        \label{sec: Selection Operators}
        \def\LexicasePTBetter{13}
        \def\LexicasePTWorse{5}
        \def\LexicaseTDCDBetter{10}
        \def\LexicaseTDCDWorse{2}
        In this section, we compare three different selection operators: dominance-based tournament selection in NSGA-II, Pareto tournament selection, and lexicase selection.
        \begin{itemize}
            \item \textbf{Dominance-based Tournament Selection (DTS)~\cite{deb2002fast}:} This method uses the dominance relationship to select individuals. Specifically, two individuals are randomly selected from the population, and the one that dominates the other is chosen. If neither dominates, one is selected randomly.
            \item \textbf{Pareto Tournament Selection (PTS)~\cite{kotanchek2007pursuing}:} The key idea of Pareto tournament selection is to select a subset of individuals from the current population, with the subset size set to 10\% of the population in this paper. These individuals are then sorted based on two objectives, and the entire first Pareto front is selected. Since the crossover operation requires pairs of individuals, the total number of selected individuals is rounded down to the nearest even number to ensure divisibility by two.
            \item \textbf{Lexicase Selection:} Lexicase selection is used as the baseline operator in this comparison. As described in \Cref{sec: Algorithm Framework}, this paper employs automatic $\epsilon$-lexicase selection, a variant tailored for regression tasks.
        \end{itemize}
        For dominance-based and Pareto tournament selection, the leave-one-out cross-validation loss and the vicinal Jensen gap are used as the two objectives, consistent with the objective values employed by the environmental selection operator. The statistical comparison of test $R^2$ scores is summarized in \Cref{tab: Selection}. The results indicate that lexicase selection outperforms dominance-based tournament selection on \LexicasePTBetter{} datasets while performing worse on \LexicasePTWorse{} datasets. Similarly, lexicase selection outperforms Pareto tournament selection on \LexicaseTDCDBetter{} datasets and underperforms on \LexicaseTDCDWorse{} datasets. These observations suggest that performance differences between the selection operators are evident but limited to a subset of datasets. Generally, dominance-based and Pareto tournament selection rely on aggregated fitness scores for optimization, which tends to favor generalists. In contrast, lexicase selection fully leverages semantic information to prioritize specialists. By doing so, lexicase selection helps identify building blocks that are useful for specific instances, which the crossover operator can subsequently combine into better solutions, thereby escaping local optima. However, when comparing training time, Pareto tournament selection is significantly faster than lexicase selection, as shown in \cref{fig: Selection}. This lower training time is not only due to the selection process of Pareto tournament being more time-efficient than lexicase selection but also because Pareto tournament selection explicitly leverages low-complexity solutions to generate offspring, thereby reducing the time required for fitness evaluation. It is important to note that lexicase selection and Pareto tournament selection are not mutually exclusive, and further investigations on their comparisons are still needed. In fact, it may be more beneficial to combine both methods through an adaptive operator selection mechanism that dynamically determines the most appropriate selection operator based on evolutionary information~\cite{zhang2023automatically}.
        \begin{figure}[!tb]
            \centering
            \includegraphics[width=0.5\linewidth]{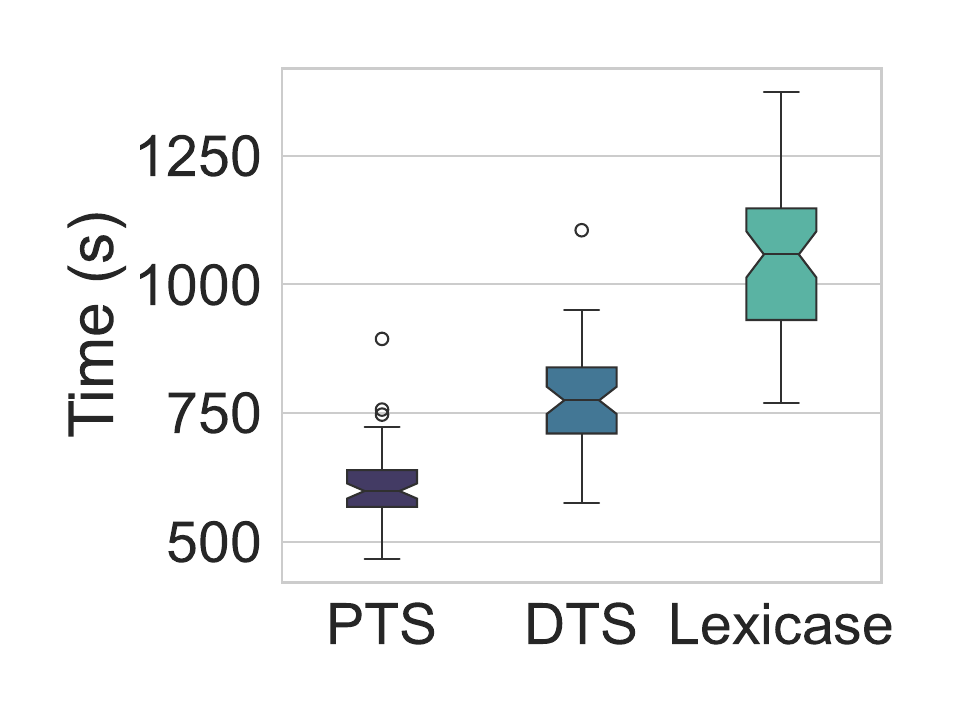}
            \caption{Comparison of using different selection operators.}
            \label{fig: Selection}
        \end{figure}
        \begin{table}[!tb]
            \centering
            \caption{Statistical comparison of \textbf{test $R^2$ scores} using different selection operators.}
            \label{tab: Selection}
            \begin{tabular}{ccccc}
                \toprule
                & \textbf{DTS}           & \textbf{Lexicase}       \\
                \midrule
                \textbf{PTS} & 0(+)/52($\sim$)/6({-}) & 5(+)/40($\sim$)/13({-}) \\
                \textbf{DTS} & ---                    & 2(+)/46($\sim$)/10({-}) \\
                \bottomrule
            \end{tabular}
        \end{table}
        \section{Parameter Analysis}
        \subsection{Alpha in Beta Distribution}
        \label{sec: Alpha}
        \def\BetaDifference{0}
        The alpha value in the Beta distribution determines the ratio of synthetic samples $\lambda$ used in synthesizing vicinal samples. Specifically, a larger alpha value indicates that synthetic samples are generated from a more balanced mix of two original samples, whereas a smaller alpha means that synthetic samples are predominantly derived from a single source sample, as illustrated in \cref{fig: Beta Distribution}. The experimental results presented in \Cref{tab: Alpha} demonstrate that the sensitivity of the $R^2$ scores to alpha is minimal. For example, increasing alpha from 10 to 100 leads to no significant change in the final $R^2$ score. Therefore, in real-world applications, an alpha value of 10 is deemed sufficient.
        \begin{table}[!tb]
            \centering
            \caption{Statistical comparison of \textbf{test $R^2$ scores} using different alpha values in the Beta distribution.}
            \label{tab: Alpha}
            \begin{tabular}{ccccc}
                \toprule
                & \textbf{100}           & \textbf{1}             \\
                \midrule
                \textbf{10}  & 0(+)/58($\sim$)/0({-}) & 0(+)/56($\sim$)/2({-}) \\
                \textbf{100} & ---                    & 1(+)/57($\sim$)/0({-}) \\
                \bottomrule
            \end{tabular}
        \end{table}
        \begin{figure}[!tb]
            \centering
            \includegraphics[width=0.65\linewidth]{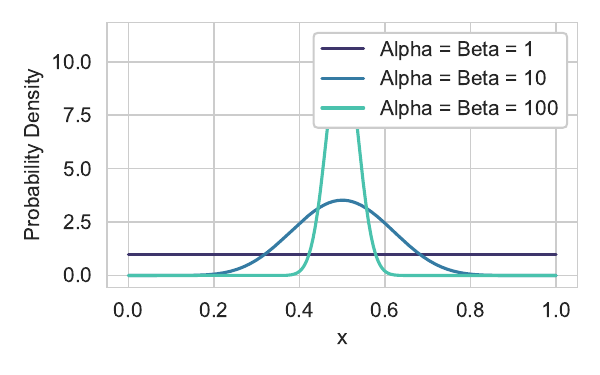}
            \caption{Probability density functions of the Beta distribution for different values of $\alpha = \beta$.}
            \label{fig: Beta Distribution}
        \end{figure}
        \subsection{Spread of Sampling}
        \def\BandwidthDifference{4}
        The spread parameter $\gamma$ determines the range for defining vicinal examples. As shown in \cref{fig: RBF Kernel}, a smaller spread parameter includes more distant examples as neighbors, applying Jensen gap regularization to broader regions. Conversely, a larger spread parameter restricts focus to closer examples, so the regularization primarily smooths over nearby samples. The experimental results for different spread parameters are presented in \Cref{tab: Bandwidth}. The test $R^2$ scores indicate that VJM is generally insensitive to the selection of the spread parameter. For instance, reducing the spread parameter from 0.5 to 0.1 results in a noticeable difference in only \BandwidthDifference{} datasets. This insensitivity suggests that, for real-world applications, a spread parameter of 0.5 is sufficient.
        \begin{table}[!tb]
            \centering
            \caption{Statistical comparison of \textbf{test $R^2$ scores} using different spread parameters.}
            \label{tab: Bandwidth}
            \begin{tabular}{ccccc}
                \toprule
                & \textbf{0.5}           & \textbf{1.0}           \\
                \midrule
                \textbf{0.1} & 3(+)/54($\sim$)/1({-}) & 1(+)/55($\sim$)/2({-}) \\
                \textbf{0.5} & ---                    & 0(+)/57($\sim$)/1({-}) \\
                \bottomrule
            \end{tabular}
        \end{table}
        \begin{figure*}[!tb]
            \centering
            \includegraphics[width=0.75\linewidth]{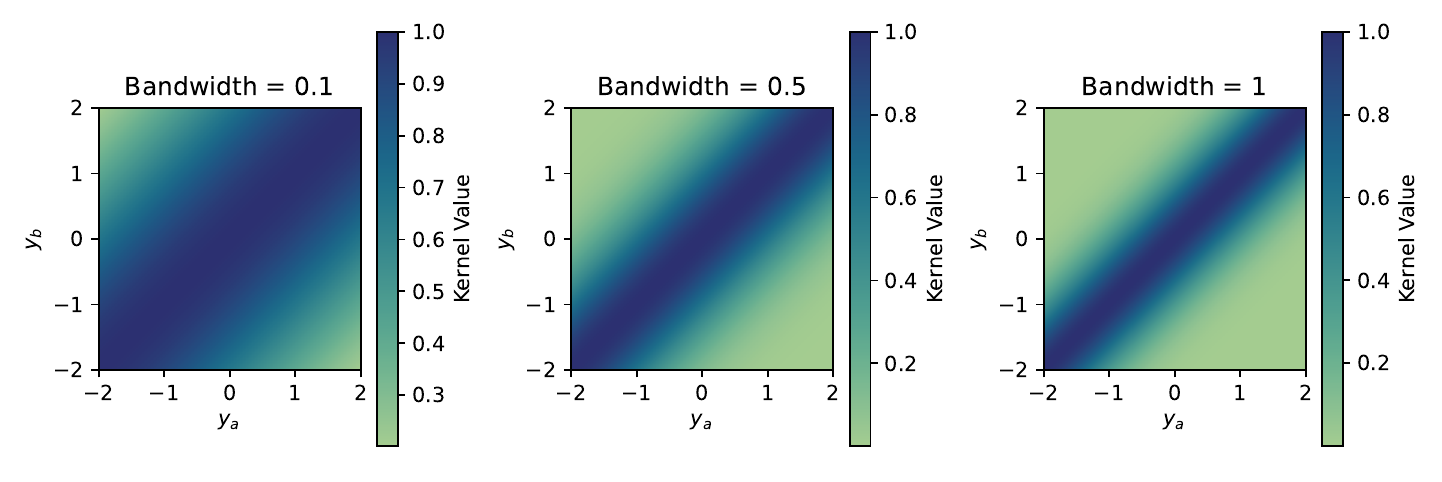}
            \caption{Kernel values for different bandwidths $\gamma$ for pairs of $y_a$ and $y_b$.}
            \label{fig: RBF Kernel}
        \end{figure*}
        \subsection{$\sigma$ in Finite Difference}
        \label{sec: Sigma}
        \def\BetterOnePointFive{9}
        \def\WorseOnePointFive{4}
        \def\BetterOneTwoPointFive{30}
        \def\WorseOneTwoPointFive{5}
        The parameter $\sigma$ in finite difference-based vicinal data synthesis determines the magnitude of noise added to the synthesized data. In this section, we analyze the sensitivity of test $R^2$ scores to different $\sigma$ values. Specifically, $\sigma$ values of 0.05, 0.1, and 0.25 are tested for comparison. The experimental results are presented in \Cref{tab: Sigma}. The results indicate that vicinal data synthesis is sensitive to the choice of $\sigma$. For instance, $\sigma = 0.1$ performs significantly better than $\sigma = 0.05$ on \BetterOnePointFive{} datasets but worse on \WorseOnePointFive{} datasets. Similarly, $\sigma = 0.1$ outperforms $\sigma = 0.25$ on \BetterOneTwoPointFive{} datasets but underperforms on \WorseOneTwoPointFive{} datasets. These results suggest that $\sigma = 0.1$ serves as a good balance point in general. Compared with the vicinal Jensen gap, finite difference regularization appears more sensitive to the hyperparameter $\sigma$. A possible explanation is that, for real-world datasets, local linearity is common. Thus, encouraging the range $[a, b]$ or $[a, b + \epsilon]$ to exhibit linearity does not make a substantial difference for small $\epsilon$. In contrast, finite difference regularization imposes a stronger constraint. Therefore, regularizing $[a,b]$ or $[a,b+\epsilon]$ to be constant can lead to more pronounced differences in outcomes.
        \begin{table}[!t]
            \centering
            \caption{Statistical comparison of \textbf{test $R^2$ scores} using different $\sigma$ values in finite difference-based synthesis.}
            \label{tab: Sigma}
            \begin{tabular}{ccccc}
                \toprule
                & \textbf{0.1}           & \textbf{0.25}           \\
                \midrule
                \textbf{0.05} & 4(+)/45($\sim$)/9({-}) & 25(+)/26($\sim$)/7({-}) \\
                \textbf{0.1}  & ---                    & 30(+)/23($\sim$)/5({-}) \\
                \bottomrule
            \end{tabular}
        \end{table}
        \begin{table*}[htbp]
            \centering
            \caption{The hyperparameter spaces of machine learning methods used in grid search~\cite{cava2021contemporary,chen2024can}.}
            \begin{tabular}{l p{45em}}
                \toprule
                Method           & Hyperparameters                                                                                                                                                                                                                                                                                          \\
                \midrule
                AdaBoost         & \{`learning\_rate': (0.01, 0.1, 1.0, 10.0), `n\_estimators': (10, 100, 1000)\}                                                                                                                                                                                                                           \\
                KernelRidge      & \{`kernel': (`linear', `poly', `rbf', `sigmoid'), `alpha': (0.0001, 0.01, 0.1, 1), `gamma': (0.01, 0.1, 1, 10)\}                                                                                                                                                                                         \\
                LGBM             & \{`n\_estimators': (10, 50, 100, 250, 500, 1000), `learning\_rate': (0.0001, 0.01, 0.05, 0.1, 0.2), `subsample': (0.5, 0.75, 1), `boosting\_type': (`gbdt', `dart', `goss')\}                                                                                                                            \\
                MLP              & \{`activation': (`logistic', `tanh', `relu'), `solver': (`lbfgs', `adam', `sgd'), `learning\_rate': (`constant', `invscaling', `adaptive')\}                                                                                                                                                             \\
                ExcelFormer      & \{`mix\_type': (`feat\_mix', `hidden\_mix', `none')\}                                                                                                                                                                                                                                                    \\
                RandomForest     & \{`n\_estimators': (10, 100, 1000), `min\_weight\_fraction\_leaf': (0.0, 0.25, 0.5), `max\_features': (`sqrt', `log2', None)\}                                                                                                                                                                           \\
                ExtraTrees       & \{`n\_estimators': (10, 100, 1000), `min\_weight\_fraction\_leaf': (0.0, 0.25, 0.5), `max\_features': (`sqrt', `log2', None)\}                                                                                                                                                                           \\
                LinearRegression & N/A                                                                                                                                                                                                                                                                                                      \\
                ElasticNet       & \{`alpha': (1e-06, 1e-04, 0.01, 1), `l1\_ratio': (0.1, 0.5, 0.7, 0.9, 0.95, 0.99, 1)\}                                                                                                                                                                                                                   \\
                LassoLars        & \{`alpha': (1e-04, 0.001, 0.01, 0.1, 1)\}                                                                                                                                                                                                                                                                \\
                SVR              & \{`C': (0.1, 1, 10, 100), `epsilon': (0.01, 0.1, 0.5, 1), `kernel': (`linear', `rbf', `poly')\}                                                                                                                                                                                                          \\
                KNeighbors       & \{`n\_neighbors': (1, 3, 5, 10), `weights': (`uniform', `distance'), `p': (1, 2)\}                                                                                                                                                                                                                       \\
                XGB              & \{`n\_estimators': (10, 50, 100, 250, 500, 1000), `learning\_rate': (0.0001, 0.01, 0.05, 0.1, 0.2), `gamma': (0, 0.1, 0.2, 0.3, 0.4), `subsample': (0.5, 0.75, 1)\}                                                                                                                                      \\
                Operon           & \{`population\_size': (500,), `pool\_size': (500,), `max\_length': (50,), `allowed\_symbols': (`add,mul,aq,constant,variable',), `local\_iterations': (5,), `offspring\_generator': (`basic',), `tournament\_size': (5,), `reinserter': (`keep-best',), `max\_evaluations': (500000,)\}                  \\
                & \{`population\_size': (500,), `pool\_size': (500,), `max\_length': (25,), `allowed\_symbols': (`add,mul,aq,exp,log,sin,tanh,constant,variable',), `local\_iterations': (5,), `offspring\_generator': (`basic',), `tournament\_size': (5,), `reinserter': (`keep-best',), `max\_evaluations': (500000,)\} \\
                & \{`population\_size': (500,), `pool\_size': (500,), `max\_length': (25,), `allowed\_symbols': (`add,mul,aq,constant,variable',), `local\_iterations': (5,), `offspring\_generator': (`basic',), `tournament\_size': (5,), `reinserter': (`keep-best',), `max\_evaluations': (500000,)\}                  \\
                & \{`population\_size': (100,), `pool\_size': (100,), `max\_length': (50,), `allowed\_symbols': (`add,mul,aq,constant,variable',), `local\_iterations': (5,), `offspring\_generator': (`basic',), `tournament\_size': (3,), `reinserter': (`keep-best',), `max\_evaluations': (500000,)\}                  \\
                & \{`population\_size': (100,), `pool\_size': (100,), `max\_length': (25,), `allowed\_symbols': (`add,mul,aq,exp,log,sin,tanh,constant,variable',), `local\_iterations': (5,), `offspring\_generator': (`basic',), `tournament\_size': (3,), `reinserter': (`keep-best',), `max\_evaluations': (500000,)\} \\
                & \{`population\_size': (100,), `pool\_size': (100,), `max\_length': (25,), `allowed\_symbols': (`add,mul,aq,constant,variable',), `local\_iterations': (5,), `offspring\_generator': (`basic',), `tournament\_size': (3,), `reinserter': (`keep-best',), `max\_evaluations': (500000,)\}                  \\
                FFX              & N/A                                                                                                                                                                                                                                                                                                      \\
                \bottomrule
            \end{tabular}
            \label{fig: SRBench ML-Hyperparameters}
        \end{table*}
        \onecolumn
        \setlength{\LTcapwidth}{\textwidth}
        \begin{longtable}[htbp]{lrrrrrrrr}
            \caption{Detailed median test $R^2$ of top-performing algorithms on all 58 PMLB datasets.} \\
            \label{tab: Detailed Results} \\
            \toprule
            & VJM-GP    & XGB       & RandomForest & ExtraTrees & MLP       & AdaBoost  & SVR       & ElasticNet \\
            \midrule
            1027 & 0.840565  & 0.826994  & 0.824332     & 0.812448   & 0.855435  & 0.788757  & 0.854015  & 0.854696   \\
            1028 & 0.334995  & 0.290876  & 0.274882     & 0.198578   & 0.334498  & 0.303251  & 0.304832  & 0.318656   \\
            1029 & 0.542249  & 0.441439  & 0.395716     & 0.293007   & 0.526766  & 0.442702  & 0.540116  & 0.543636   \\
            1030 & 0.345114  & 0.238939  & 0.247096     & 0.190406   & 0.341778  & 0.239396  & 0.326406  & 0.345797   \\
            1089 & 0.726869  & 0.647972  & 0.709673     & 0.715010   & 0.735213  & 0.624932  & 0.747592  & 0.740863   \\
            1096 & 0.825504  & 0.704897  & 0.707586     & 0.550733   & 0.753582  & 0.694478  & 0.813204  & 0.788902   \\
            1191 & 0.198603  & 0.182081  & 0.201308     & 0.189125   & 0.140607  & 0.179012  & 0.174385  & 0.172814   \\
            1193 & 0.529916  & 0.516908  & 0.536592     & 0.530689   & 0.518055  & 0.527700  & 0.504949  & 0.540347   \\
            1196 & 0.407728  & 0.344633  & 0.370030     & 0.356128   & 0.398170  & 0.357886  & 0.410797  & 0.420361   \\
            1199 & 0.385372  & 0.346492  & 0.363977     & 0.361324   & 0.377553  & 0.345141  & 0.370149  & 0.378311   \\
            1201 & -0.000504 & -0.061526 & 0.018752     & 0.017796   & -0.006524 & -0.045104 & 0.003838  & 0.016853   \\
            1203 & 0.465604  & 0.459101  & 0.481883     & 0.472472   & 0.447231  & 0.504761  & 0.432235  & 0.466732   \\
            1595 & -0.050663 & -0.080785 & -0.010936    & -0.009038  & -0.006687 & -0.061970 & -0.024205 & -0.005064  \\
            192  & 0.492727  & 0.493044  & 0.487049     & 0.370296   & 0.560498  & 0.564445  & 0.434034  & 0.368128   \\
            195  & 0.816963  & 0.830318  & 0.868035     & 0.854877   & 0.747347  & 0.801928  & 0.741180  & 0.743558   \\
            197  & 0.894962  & 0.941402  & 0.935246     & 0.954910   & 0.869832  & 0.944216  & 0.357505  & 0.544948   \\
            201  & 0.767137  & 0.816091  & 0.778009     & 0.844249   & 0.465433  & 0.737237  & 0.676101  & 0.190848   \\
            207  & 0.751995  & 0.829024  & 0.852168     & 0.861556   & 0.759248  & 0.815315  & 0.729024  & 0.748505   \\
            210  & 0.809208  & 0.768884  & 0.796473     & 0.793094   & 0.845722  & 0.738728  & 0.886208  & 0.871676   \\
            215  & 0.918329  & 0.905344  & 0.896656     & 0.884696   & 0.882206  & 0.888312  & 0.818823  & 0.674278   \\
            218  & 0.423724  & 0.414668  & 0.435067     & 0.450343   & 0.313891  & 0.404602  & 0.267256  & 0.213272   \\
            225  & 0.340687  & 0.484648  & 0.467970     & 0.535939   & 0.329063  & 0.474318  & 0.346464  & 0.334667   \\
            227  & 0.893457  & 0.928316  & 0.923821     & 0.949588   & 0.857528  & 0.933657  & 0.371028  & 0.530435   \\
            228  & 0.551606  & 0.683940  & 0.638867     & 0.524560   & 0.656204  & 0.618829  & 0.511206  & 0.583576   \\
            229  & 0.856066  & 0.832123  & 0.811026     & 0.798169   & 0.783388  & 0.802321  & 0.742188  & 0.756406   \\
            230  & 0.794644  & 0.819710  & 0.715052     & 0.858531   & 0.737586  & 0.729656  & 0.746859  & 0.743686   \\
            294  & 0.748777  & 0.746761  & 0.765740     & 0.801421   & 0.763300  & 0.752101  & 0.802043  & 0.662896   \\
            344  & 0.996209  & 0.951705  & 0.942087     & 0.982543   & 0.994545  & 0.927896  & 0.938874  & 0.795322   \\
            4544 & 0.691813  & 0.585010  & 0.532841     & 0.599908   & 0.526252  & 0.580982  & 0.552433  & 0.694516   \\
            485  & 0.454381  & 0.508044  & 0.458570     & 0.425216   & 0.424676  & 0.446619  & 0.307905  & 0.391354   \\
            503  & 0.698187  & 0.676443  & 0.663319     & 0.674810   & 0.702977  & 0.638516  & 0.711517  & 0.721085   \\
            505  & 0.996280  & 0.985465  & 0.978513     & 0.990148   & 0.993126  & 0.978002  & 0.995344  & 0.994654   \\
            519  & 0.724352  & 0.686646  & 0.620122     & 0.610716   & 0.734809  & 0.701810  & 0.735497  & 0.738487   \\
            522  & 0.080212  & 0.241097  & 0.219749     & 0.225801   & 0.102416  & 0.193508  & 0.115079  & 0.065844   \\
            523  & 0.935079  & 0.942572  & 0.947106     & 0.931709   & 0.928203  & 0.941711  & 0.920709  & 0.936675   \\
            527  & 0.956401  & 0.681786  & 0.729301     & 0.746947   & 0.784684  & 0.700967  & 0.933900  & 0.751345   \\
            529  & 0.761634  & 0.606867  & 0.534707     & 0.587069   & 0.776545  & 0.514274  & 0.775769  & 0.781650   \\
            537  & 0.597876  & 0.536337  & 0.510514     & 0.524523   & 0.617349  & 0.492070  & 0.572244  & 0.591215   \\
            542  & 0.474008  & 0.301537  & 0.337266     & 0.376004   & 0.249359  & 0.393924  & 0.251514  & 0.330857   \\
            547  & 0.447828  & 0.473047  & 0.478522     & 0.479872   & 0.451461  & 0.479329  & 0.441376  & 0.461079   \\
            556  & 0.795224  & 0.838565  & 0.845691     & 0.793570   & 0.738266  & 0.870867  & 0.768854  & -0.003486  \\
            557  & 0.867248  & 0.857777  & 0.849656     & 0.824656   & 0.794932  & 0.848236  & 0.771608  & -0.001402  \\
            560  & 0.982260  & 0.964594  & 0.967752     & 0.973442   & 0.967967  & 0.950733  & 0.985790  & 0.978617   \\
            561  & 0.959809  & 0.887257  & 0.734116     & 0.822712   & 0.936104  & 0.708708  & 0.796391  & 0.794266   \\
            562  & 0.893457  & 0.928316  & 0.923821     & 0.949588   & 0.857528  & 0.933657  & 0.371028  & 0.530435   \\
            564  & 0.901678  & 0.779753  & 0.655387     & 0.695884   & 0.699820  & 0.680869  & 0.715631  & 0.696366   \\
            573  & 0.894962  & 0.941402  & 0.935246     & 0.954910   & 0.869832  & 0.944216  & 0.357505  & 0.544948   \\
            574  & 0.134109  & 0.297048  & 0.307811     & 0.320709   & 0.055241  & 0.187690  & 0.224135  & 0.056789   \\
            659  & 0.604900  & 0.394315  & 0.543908     & 0.407758   & 0.498848  & 0.516694  & 0.659796  & 0.546316   \\
            663  & 0.997748  & 0.995491  & 0.976191     & 0.987455   & 0.998490  & 0.955134  & 0.998191  & 0.972364   \\
            665  & 0.316847  & 0.155599  & 0.142349     & 0.115811   & 0.292470  & 0.313316  & 0.298176  & 0.291333   \\
            666  & 0.536075  & 0.498181  & 0.473588     & 0.538856   & 0.526476  & 0.500480  & 0.471409  & 0.529492   \\
            678  & 0.276509  & 0.115992  & 0.187108     & 0.166884   & 0.214906  & 0.141751  & 0.253331  & 0.258067   \\
            687  & 0.457557  & 0.425380  & 0.465336     & 0.439297   & 0.387792  & 0.446510  & 0.308491  & 0.383927   \\
            690  & 0.958719  & 0.966728  & 0.963135     & 0.967320   & 0.969688  & 0.955554  & 0.970536  & 0.892017   \\
            695  & 0.841782  & 0.806271  & 0.823695     & 0.841566   & 0.855735  & 0.822629  & 0.862782  & 0.866051   \\
            706  & 0.628841  & 0.542783  & 0.564594     & 0.619290   & 0.607258  & 0.547730  & 0.515102  & 0.582595   \\
            712  & 0.746153  & 0.740555  & 0.717430     & 0.689963   & 0.746347  & 0.721980  & 0.738814  & 0.747513   \\
            \bottomrule
        \end{longtable}
        \clearpage
        \twocolumn
\end{document}